\documentclass[twoside]{article}

\usepackage[accepted]{aistats2026}
\usepackage[round]{natbib}

\usepackage[colorlinks=true, citecolor=blue, linkcolor=blue, urlcolor=magenta]{hyperref}
\usepackage{amsmath, amssymb, amsfonts}
\usepackage{amsthm}
\usepackage{graphicx}
\usepackage{enumitem}
\usepackage{booktabs}
\usepackage{etoc}
\let\oldaddcontentsline\addcontentsline

\theoremstyle{plain}
\newtheorem{theorem}{Theorem}[section]
\newtheorem{lemma}{Lemma}[section]
\newtheorem{proposition}{Proposition}[section]

\theoremstyle{definition}
\newtheorem{definition}{Definition}[section]
\newtheorem{assumption}{Assumption}[section]

\theoremstyle{remark}
\newtheorem{remark}{Remark}[section]
\usepackage{xcolor}

\begin{document}
\renewcommand{\addcontentsline}[3]{}

%
\runningtitle{Impact of PE: Clean and Adversarial RC for Transformers under In-Context Regression}

%

\twocolumn[

\aistatstitle{Impact of Positional Encoding: Clean and Adversarial Rademacher Complexity for Transformers under In-Context Regression}

\aistatsauthor{ Weiyi He \And Yue Xing }

\aistatsaddress{ Michigan State University \And Michigan State University } ]

\begin{abstract}
Positional encoding (PE) is a core architectural component of Transformers, yet its impact on the Transformer's generalization and robustness remains unclear. In this work, we provide the first generalization analysis for a single-layer Transformer under in-context regression that explicitly accounts for a completely trainable PE module. Our result shows that PE systematically enlarges the generalization gap. Extending to the adversarial setting, we derive the adversarial Rademacher generalization bound. We find that the gap between models with and without PE is magnified under attack, demonstrating that PE amplifies the vulnerability of models. Our bounds are empirically validated by a simulation study. Together, this work establishes a new framework for understanding the clean and adversarial generalization in ICL with PE.
\end{abstract}

\section{Introduction}
Large language models (LLMs), built upon the basic Transformer architecture \citep{vaswani2017attention, brown2020language,radford2019language}, have demonstrated remarkable capabilities of in-context learning (ICL), allowing them to perform new tasks based on a few examples provided in a prompt, without any parameter updates \citep{brown2020language, dong2022survey}. This powerful few-shot learning paradigm has shown impressive empirical performance across various complex tasks \citep{liang2022holistic, raventos2023pretraining}.

The success of ICL has motivated several theoretical studies to explain the underlying mechanism of ICL \citep{li2023transformers, deng2023attention, zhang2024trained}. These theoretical works often focus on understanding how Transformers can learn from examples by framing the process as an implicit implementation of classic algorithms \citep{von2023transformers, ahn2023transformers, xie2021explanation}. However, these analyses often rely on key simplifications. For example, the model in \cite{cui2024superiority} only considers the basic attention score in their formulation. The Positional Encoding (PE), which is an important component for encoding positional information, is missing. While some empirical works have investigated properties of PEs \citep{gao2024pe, ccavcsi2025fpe}, a formal generalization theory is still lacking. 

Furthermore, the urgency of filling this theoretical gap is also underscored by the adversarial vulnerability of LLMs. Despite the power of ICL, these models are highly vulnerable to carefully crafted adversarial attacks that can degrade performance or bypass safety alignments \citep{shayegani2023survey, zou2023universal}. For instance, the jailbreaking attack can trick the models into generating harmful or toxic content \citep{wei2023jailbroken}, while prompt injection can hijack the model's behavior for malicious purposes \citep{liu2023prompt}. Other literature also develops poisoning attacks empirically on ICL examples, e.g., \cite{he2025data}.

Observing the above limitations, this work provides a unified framework for ICL generalization based on Rademacher complexity (RC). Our framework extends the foundational theory of ICL to two crucial dimensions: the impact of trainable PEs on the model complexity, and the model's adversarial generalization under adversarial attacks.

Our contributions are summarized as follows:

First, we provide the generalization bound for a family of one-layer Transformers, \textbf{in the ICL regression setting,} that explicitly includes a completely trainable PE module. To address the technical challenge of analyzing PE parameters that are deeply embedded within the non-linear self-attention mechanism, prior theoretical work often absorbs these parameters into a single, uniform norm bound \citep{edelman2022inductive,li2025theoretical}. Our analysis quantifies the increase in the model's Rademacher complexity from PE.

Second, we extend our analysis to adversarial generalization by deriving the Adversarial Rademacher Complexity (ARC) bounds for the Transformer. Our approach builds on the work of \citet{xiao2022adversarial}, which connects robustness to the covering number of the function class. A technical challenge is that the core tools in \citet{xiao2022adversarial} are designed for classification tasks. To handle the MSE loss in ICL regression, we employ the surrogate loss method, a general technique to bound the complex adversarial loss with a simpler function \citep{khim2018adversarial,yin2019rademacher, gao2021theoretical}. Our derived bounds reveal that the difference between the bounds with/without PE is larger under adversarial attack, indicating that Transformers with PE are more vulnerable.

Finally, we validate our theoretical findings with experiments. The results show that Transformers with a trainable PE have a consistently larger generalization gap and are more sensitive to adversarial attacks on in-context examples. We also observe that the gap increases with stronger attacks but decreases as the context length grows, confirming that longer contexts still improve robustness even under adversarial attacks.

\section{Related Works}
\textbf{Theoretical Analysis of ICL.}
Many theoretical works aim to demystify the mechanisms behind ICL. \citet{garg2022can} models the Transformer as an implicit algorithm that performs in-context learning of linear functions. Subsequent works have extended this framework with different focuses. \citet{xing2024theoretical} studied the role of each component in the transformer architecture to learn from the unstructured data. The work of \citet{cui2024superiority} demonstrated that multi-head attention with a substantial embedding dimension performs better than single-head attention. Other related works can be found in \citet{min2021metaicl,chen2022improving,zhang2022active}. This line of work typically involves analyzing a sufficiently large amount of data and examining how the Transformer performs in the ideal case.

\textbf{Generalization Bounds for Transformers.}
The generalization ability of Transformers has been studied through various theoretical perspectives. One line of work relies on the uniform bounds for the input data matrix. Their bounds may depend on the length of the sequence \citep{edelman2022inductive} or be independent of it under a certain format of covering numbers \citep{trauger2024sequence}. Other approaches focus on specific settings, such as margin-based bounds for binary classification \citep{wei2022statistically}. In contrast, our work provides an analysis specifically tailored to the ICL setting. By adopting a specific Gaussian assumption, we can derive a more fine-grained bound that explicitly captures how the generalization gap decays with the context length $t$.

\textbf{Adversarial Generalization.}
Adversarial generalization is a broad field that utilizes various tools to provide theoretical insights. These include methods based on VC-dimension \citep{montasser2019vc, attias2022improved}, algorithmic stability \citep{xing2021algorithmic, xiao2022stability}, and distributional robustness \citep{sinha2017certifiable}. Moreover, \citet{xiao2022adversarial} provided the bound of Adversarial Rademacher Complexity (ARC) for deep neural networks based on the covering numbers. However, the adversarial generalization for the Transformers on ICL based on ARC remains unexplored, and our work provides the ARC results for the considered ICL setting.

\section{Preliminaries}
\subsection{ICL Task and Transformer Architecture}
To mathematically define ICL, instead of merely passing a query $x_q\in \mathbb{R}^d$ to the transformer to make a prediction, ICL passes a prompt, i.e., a few examples with their labels $\{(x_i,y_i)\}_{i=1,...,t}$ together with the query $x_q$, to the transformer. Using the prompt in the format of 
\begin{equation}
\label{eq:input format}
X = 
\begin{pmatrix}
x_1 & x_2 & \dots & x_t & x_q \\
y_1 & y_2 & \dots & y_t & 0
\end{pmatrix}^\top
\in \mathbb{R}^{(t+1)\times(d+1)}.
\end{equation}
The transformer can learn from the examples to infer the prediction for $x_q$. 

\label{sec:TF_structure}
Following \citet{trauger2024sequence}, we define the simplified single-layer, single-head Transformer architecture as follows. Let the input sequence be $X \in \mathbb{R}^{(t+1)\times (d+1)}$ as defined in Eq. \eqref{eq:input format}. The model parameters $\theta$ consist of the trainable weight matrices $W_{\text{in}}\in \mathbb{R}^{(d+1)\times d_m}$, $W_Q, W_K, W_V \in \mathbb{R}^{d_m\times d_m}$, and $W_c\in \mathbb{R}^{d_m}$, where $d_m$ is the embedding dimension. Following \citet{edelman2022inductive}, we bound these matrices by some constants $\|W_c^T\|_{1,\infty} \leq B_{W_c}, \|W_v^T\|_{1,\infty} \leq B_{W_v}$. 

The model first computes content embeddings via a read-in matrix $W_{\text{in}}$. To provide the model with sequence order information, we add a PE matrix $P \in \mathbb{R}^{(t+1) \times d_m}$. The final input representation $H$ is the sum of content and positional embeddings:
\begin{equation}
\label{eq:PE}
H = X W_{\text{in}} + P.
\end{equation}
In this work, we specifically analyze the case where $P$ is a completely trainable parameter and $\|P\|_F\leq B_P$.

For simplicity, we combine the query and key matrices into $W_{QK} = W_Q W_K^\top \in \mathbb{R}^{d_m\times d_m}$. Let $\sigma$ be an $L_{\sigma}$-Lipschitz activation function. The attention head computes an output representation matrix ${H}' \in \mathbb{R}^{(t+1) \times d_m}$ as:
\begin{equation}
\label{eq:self_attention}
{H}' = \sigma \left(\text{RowSoftmax}\left(\frac{H W_{QK} H^\top}{\sqrt{d_m}}\right) H W_V\right).
\end{equation}
The final prediction of the model, $f_\theta$, is a linear readout from the last row of ${H}'$, which corresponds to the query token. Let ${h}'_q \in \mathbb{R}^{d_m}$ be the last row of ${H}'$. The prediction is given by $f_\theta(X) = W_c ^\top {h}'_q$.

\subsection{Training and Adversarial Loss}
We consider a Transformer parameterized by $\theta \in \Theta$. The model is trained to minimize the standard Mean Squared Error (MSE) loss. For a given prompt containing context $S_t = \{(x_i, y_i)\}_{i=1}^t$ and a query $x_q$, the model produces a prediction $\hat{y}_q$. Then we minimize the expected risk:
\begin{equation}
\label{eq:loss}
L(\theta) = \mathbb{E}_{S_t, x_q, y_q}\left(\hat{y}_q(S_t, x_q; \theta)- y_q \right)^2.
\end{equation}

In the adversarial setting, a perturbed input matrix $X'$ is constructed as:
\begin{equation}
\label{eq:perturbed_input_format}
X' = 
\begin{pmatrix}
x'_1 & x'_2 & \dots & x'_t & x'_q \\
y_1 & y_2 & \dots & y_t & 0
\end{pmatrix}^\top
\in \mathbb{R}^{(t+1)\times(d+1)},
\end{equation}
where each $x'_i$ is a perturbed version of the original $x_i$ and the attacked prompt is then $S'_t=\{(x'_i,y_i)\}_{i=1}^t$. We consider a global threat model where the total perturbation on the input matrix is bounded by $\|X' - X\|_F \le \varepsilon$. This defines the formal definition of the adversarial loss:
\begin{equation}
\label{eq:adv_loss}
\tilde{\ell}_\theta(X, y_q) \triangleq \sup_{X' \text{ s.t. } \|X' - X\|_F \le \varepsilon} \left( \hat{y}_q(X'; \theta) - y_q \right)^2,
\end{equation}
where $\hat{y}_q(X';\theta)$ denotes the model's prediction for the query based on the entire perturbed input matrix $X'$.

\subsection{Generalization Framework}
Our theoretical analysis is based on the standard framework of generalization bounds. We aim to bound the generalization gap, $\sup_{h \in \mathcal{H}} ( R(h) - \hat{R}_{S}(h) )$, where $R(h)=\mathbb{E}_{(x,y)}[\ell(h(x), y)]$ is the true risk and $\hat{R}_{S}(h)=\frac{1}{m}\sum_{i=1}^m \ell(h(x_i), y_i)$ is the empirical risk over a sample $S$ of size $m$. The primary tool is the Rademacher complexity, which quantifies the capacity of a function class. The following result connects the two concepts together \citep{shalev2014understanding}:
\begin{proposition}[Generalization Bound]
Let $\mathcal{H}$ be a hypothesis class and $\ell$ be a loss function. If the magnitude of our loss function is bounded above by $c$, with probability at least $1-\delta$:
\begin{equation*}
\sup_{h \in \mathcal{H}} \left( R(h) - \hat{R}_S(h) \right) \le 2\operatorname{Rad}_m(\ell \circ \mathcal{H}) + 4c\sqrt{\frac{2\log(4/\delta)}{m}},
\end{equation*}
where $\ell \circ \mathcal{H}$ is the class of functions formed by composing the loss with the hypotheses.
\end{proposition}
We provide the remaining definitions in Appendix~\ref{sec:appendix_definitions}.

\section{Main Results}
In this section, we present the standard and robustness generalization bound for the Transformer. We first state the data generation model as follows.
\begin{assumption}[\textbf{Data Generation Model}]
\label{assump:data_model}
In each prompt, the example pairs $\{(x_i,y_i)\}_{i=1}^t$ and the query $(x_q,y_q)$ are i.i.d. samples from the following noiseless linear regression model:
\begin{itemize}[noitemsep, topsep=0pt, leftmargin=*]
    \item The input $x \sim \mathcal{N}(0, I_d)$.
    \item The output $y = \mu^\top x$.
    \item The true coefficients $\mu \in \mathbb{R}^d$ are the same for all samples within a prompt but are drawn independently for each prompt from $\mu \sim \mathcal{N}(0, I_d/d)$.
\end{itemize}
\end{assumption}

The Assumption \ref{assump:data_model} follows from \citet{cui2024superiority} and \citet{zhang2024trained} for the data generation model. This standard setting creates a clean and analytically tractable environment, allowing us to focus on the mechanism of in-context learning itself. Our proof can also be extended to more general sub-Gaussian distributions (see Appendix~\ref{app:extention_sub_gaussian} for details).

\subsection{RC for Transformer without PE}
\label{sec:RC_for_TF}
In this section, we derive the generalization bound for a one-layer Transformer in the standard ICL setting. The analysis is grounded in the Dudley integral framework, which connects the Rademacher complexity of a function class to its geometry via covering numbers.

We first present the final theorem:
\begin{theorem}[RC for Transformer without PE]
\label{thm:rc_no_pe_final}
Under Assumption \ref{assump:data_model}, using the architecture defined in Section \ref{sec:TF_structure}, let the context $S_t$ be sampled i.i.d. from the data distribution. Let $W^*_{QK}(S_t)$ be the ideal attention matrix for the context $S_t$, and $W^*_{QK}$ be a global ideal attention matrix for the transformer. Assuming that the trained model's attention matrix $W_{QK}$ is close to $W^*_{QK}$, that is $\|W_{QK}-W^*_{QK}\|_F\leq \gamma$, for a constant $\gamma>0$. Then, with probability at least $1-\delta$ over the choice of $S_t$, there exists a constant $L_f$ and a tolerance $r$ (dependent of $\gamma$) such that the Rademacher complexity of the $\mathcal{F}_{S_t}$ is bounded by 
\begin{equation}
\operatorname{Rad}_S(\mathcal{F}_{S_t}) \lesssim O\left( \frac{L_f \sqrt{rD} \sqrt{\log t/r}}{\sqrt{m t}} \right),
\end{equation}
where $D=d\cdot d_m+d_m^2$.
\end{theorem}
This theorem provides the generalization bound for ICL regression that explicitly shows a decay with context length $t$. In contrast to prior work that provides sequence-length-independent bounds \citep{trauger2024sequence}, our result captures the core learning dynamics of ICL. 

There are multiple steps to prove Theorem \ref{thm:rc_no_pe_final}. We start from the following lemma.
\begin{lemma}[Generalization Bound via Dudley's Integral]
\label{lemma:RC_dudley}
Let $\mathcal{F}_{S_t}$ be the function class of a Transformer constrained by context $S_t$. Then its empirical Rademacher complexity is bounded by: 
\begin{equation} 
\label{eq:start_point}
\operatorname{Rad}_S(\mathcal{F}_{S_t}) \lesssim \frac{1}{\sqrt{m}} \int_0^{\operatorname{Diam}(\mathcal{F}_{S_t})} \sqrt{\log N(\mathcal{F}_{S_t}, \|\cdot\|_{m,2}, \varepsilon)} d\varepsilon,
\end{equation}
where $\|\cdot\|_{m,2}$ is the empirical $2$-norm and $N(\cdot)$ is the covering number. The diameter of a function class $\mathcal{F}$ is defined as $\operatorname{Diam}(\mathcal{F}) \triangleq \sup_{f_1, f_2 \in \mathcal{F}} \|f_1 - f_2\|_{m,2}$.
\end{lemma}

While Lemma \ref{lemma:RC_dudley} provides a general bound, the geometry of the function class $\mathcal{F}_{S_t}$ is complicated to analyze directly. To overcome this challenge, we first apply standard contraction results to peel out the $W_c$, $\sigma$, and $W_V$. This reduces the problem to the attention block that only depends on $(W_{\text{in}}, W_{QK})$. Our key insight is to analyze the effective parameter space $\Theta_{S_t}$, as defined in Definition \ref{def:effective_param_space}. Finally, we obtain the corresponding covering number of $\Theta_{S_t}$ in Lemma \ref{lemma:CoveringNumber_bound}, and connect it to the ICL solution space via Lemma \ref{lemma:para_bound} and Lemma \ref{lemma:L_G_and_r}.

\begin{definition}[Effective Parameter Space]
\label{def:effective_param_space}
Under the conditions of Theorem \ref{thm:rc_no_pe_final}, we define the effective parameter space as
\begin{equation*}
\Theta_{S_t} := \Big\{ (W_{\text{in}}, W_{QK}) : \|W_{QK} - W^*_{QK}(S_t)\|_F \le \gamma_{\text{eff}} \Big\},
\end{equation*}
Let $\Delta_t := \|W^*_{QK}(S_t)-W^*_{QK}\|_F$, as $\Delta_t \lesssim O(\sqrt{d/t})$ with high probability (see Appendix~\ref{app:proof_of_def}), for any fixed $t_{\min}$ and all $t \ge t_{\min}$, we can set constant
$\gamma_{\text{eff}} = \gamma + \Delta^*$ and $\Delta^* := \sup_{t \ge t_{\min}} \Delta_t$. 
\end{definition}

With the above definition, we can formally define $\mathcal{F}_{S_t}:=\{f_\theta:\theta\in \Theta_{S_t}\}$. Then we can connect the properties of $\Theta_{S_t}$ back to the $\mathcal{F}_{S_t}$ via a standard Lipschitz continuity assumption, and obtain the covering number of $\mathcal{F}_{S_t}$ as in the following lemma:

\begin{lemma}[Covering Number of the Parameter Space]
\label{lemma:CoveringNumber_bound}
Under the conditions of Theorem \ref{thm:rc_no_pe_final} and Definition \ref{def:effective_param_space}. Assume that for any $\theta_1, \theta_2 \in \Theta$, there is
\begin{equation}
\|f_{\theta_1} - f_{\theta_2}\|_{m,2} \leq L_f \|\theta_1 - \theta_2\|_2,
\end{equation}
where $L_f$ is a Lipschitz constant determined by the architecture of Transformer (Appendix~\ref{app:proof_covering_bound}). As the definition of $\Theta_{S_t}$, the total number of parameters is therefore $D = d \cdot d_{m} + d_{m}^2$. Let $D_t = \operatorname{Diam}(\Theta_{S_t})$, the covering number is bounded by:
\begin{equation}
\label{eq:para_bound}
\log N(\Theta_{S_t}, \|\cdot\|_2, \varepsilon) \leq D \log\left( 3 D_t/{\varepsilon} \right),
\end{equation}
where the diameter for a set $W$ is defined as $\operatorname{Diam}(W) = \sup \left\{ \| w_1 - w_2 \|_2 : w_1, w_2 \in W \right\}.$
\end{lemma}
A direct consequence is that the geometry of $\mathcal{F}_{S_t}$ can be controlled by the geometry of $\Theta_{S_t}$:
\begin{align}
\label{eq:covering_translation}
N(\mathcal{F}_{S_t}, \|\cdot\|_{m,2}, \varepsilon) &\le N(\Theta_{S_t}, \|\cdot\|_2, \varepsilon / L_f), \nonumber \\
\operatorname{Diam}(\mathcal{F}_{S_t}) &\le L_f \cdot \operatorname{Diam}(\Theta_{S_t}).
\end{align}
Inequalities \eqref{eq:covering_translation} allow us to bound the Dudley integral by analyzing the parameter space, whose covering number can be bounded using standard results.

After obtaining the covering number w.r.t. $\Theta_{S_t}$, the final step is to connect $\Theta_{S_t}$ and its diameter $D_t$ back to ICL. To do this, following recent ICL literature, we introduce the concept of an effective linear weight, $w_{\text{eff}}$, which approximates the Transformer's ability as a simple linear model. To be specific, we can take the Transformer's ICL process for linear regression as finding an effective linear weight $w_{\text{eff}}$ such that the prediction for a query $x_q$ is approximately $w_{\text{eff}}^\top x_q$, i.e.,
\begin{equation}
\label{eq:w_eff}
w_{\text{eff}} := \arg\min_{w\in \mathcal{R}^d} \mathbb{E}_{x_q \sim N(0,I_d)} \left( \hat{y}_q - w^\top x_q \right)^2.
\end{equation} 
Given the above definition, we are interested in the solution space $W_{S_t}$, which is the set of all such $w_{\text{eff}}$ generated by the allowable parameters in $\Theta_{S_t}$. We formalize the geometry of $W_{S_t}$ as follows.

\begin{lemma}
\label{lemma:para_bound}
Under the conditions of Theorem \ref{thm:rc_no_pe_final}, suppose the model is on a high-probability "good event" $\mathcal{G}_t$ where the data matrix $X_t \in \mathbb{R}^{t \times d}$ is well-conditioned as justified in Appendix~\ref{app:tail_event}. Given the solution space $W_{S_t}$ and the model's parameter space $\Theta_{S_t}$, we denote the mapping $G:\theta \mapsto w_{\text{eff}}$. Assume $G$ is inverse Lipschitz continuous where there exists a constant $L_G > 0$, independent of $t$, such that for any $\theta_1, \theta_2 \in \Theta_{S_t}$:
\begin{equation}
\label{eq:inv_lip}
\|\theta_1 - \theta_2\|_2 \leq L_G \|G(\theta_1) - G(\theta_2)\|_2.
\end{equation}
If we model the set of effective linear weights $W_{S_t}$ as a tolerance ellipsoid:
\begin{equation}
\label{eq:ellipsoid_definition}
W_{S_t} = \left\{ w: (w - \hat{w})^\top (X_t^\top X_t) (w - \hat{w}) \leq r \right\},
\end{equation}
where $\hat{w} \in \mathbb{R}^d$ is the OLS solution and $r > 0$ is a tolerance constant. Then for $t>d$,
\begin{equation}
\label{eq:diameter_contraction}
\operatorname{Diam}(\Theta_{S_t})\lesssim L_G \operatorname{Diam}(W_{S_t}) \lesssim O\left(L_G\sqrt{r/t}\right).
\end{equation}
\end{lemma}
Finally we provide some justifications on the existence of $L_G$ and the size of $r$ in Eq. \eqref{eq:diameter_contraction}: 
\begin{lemma}[Justification for $L_G$ and $r$]
\label{lemma:L_G_and_r}
Under the conditions of Theorem \ref{thm:rc_no_pe_final} and Lemma \ref{lemma:para_bound}, assume that $G$ is $C^1$ on a compact set $K\in \Theta_{S_t}$ and there exist constants $c_G, c_U>0$ such that for all $\theta\in K$ the Jacobian $J_G(\theta)$ has full row rank $d$ and satisfies $c_G \le\sigma_{\min}(J_G(\theta)) \le \sigma_{\max}(J_G(\theta)) \le c_U$. Then $L_G \leq M(c_G,c_U)$. Moreover, there exists a constant $C_r>0$ (independent of $t$ and related to $\gamma_{\text{eff}}$) and one can choose a tolerance $r$ such that $r\leq C_r$.
\end{lemma}

See the justification of all the lemmas in Appendix~\ref{app:proof_of_lemmas} and the full proof of Theorem \ref{thm:rc_no_pe_final} in Appendix~\ref{app:proof_of_standard_rc}.

\subsection{RC for Transformer with PE}
\label{sec:rc_with_pe}
While Section \ref{sec:RC_for_TF} gives a bound for a Transformer without PE, the standard architectures universally include PEs to handle sequence order. Although in the context of ICL, the order of $(x_i,y_i)$ in the prompt does not influence the solution for linear regression, adding PEs in the model increases the number of parameters, thus increasing the parameter space, and our focus is on quantifying the potential negative impact of PE on the generalization bound. Given the above, our next step is to understand how incorporating completely trainable parameters impacts the generalization bounds.

To consider PE, we employ a function class decomposition. We model any function $f_{\text{PE}}$ from the class $\mathcal{F}_{\text{PE}}$ as the sum of a function from the no-PE class $\mathcal{F}_{\text{NoPE}}$ and a bias term $f_{\text{bias}}$:
\begin{equation*}
f_{\text{PE}} = f_{\text{NoPE}} + f_{\text{bias}},
\end{equation*}
where $f_{\text{bias}}$ belongs to the class $\mathcal{F}_{\text{bias}} \triangleq \{f_{\text{PE}} - f_{\text{NoPE}}\}$. This implies $\mathcal{F}_{\text{PE}} \subseteq \mathcal{F}_{\text{NoPE}} + \mathcal{F}_{\text{bias}}$. The subadditivity of the Rademacher complexity then allows us to bound the total complexity as:
\begin{equation}
\label{eq:subadditivity}
\operatorname{Rad}_S(\mathcal{F}_{\text{PE}}) \le \operatorname{Rad}_S(\mathcal{F}_{\text{NoPE}}) + \operatorname{Rad}_S(\mathcal{F}_{\text{bias}}).
\end{equation}

This decomposition simplifies the analysis by only bounding the complexity of the bias class, $\operatorname{Rad}_S(\mathcal{F}_{\text{bias}})$, which depends on the magnitude of the functional difference introduced by the PE. We formalize this with the following assumption.

\begin{assumption}[Bounded PE Effect]
\label{assump:pe_effect}
Let $w_{\text{PE}}, \eta_{\text{PE}}$ and $w_{\text{NoPE}}, \eta_{\text{NoPE}}$ be the effective linear weights and error for models with and without PE, respectively. We assume the difference for the weight is bounded by a constant $C_{PE} > 0$:
\begin{equation}
\label{eq:pe_effect}
\|w_{\text{PE}} - w_{\text{NoPE}}\|_2 \le C_{PE}. 
\end{equation}
\end{assumption}
The above assumption supposes that adding the bounded, trainable PE matrix to the input embeddings results in a correspondingly bounded change to the model's effective linear weight. 

To be specific, the function class $\mathcal{F}_{\text{bias}}$ represents the difference in the models' outputs, $f_\text{PE}(X) - f_\text{NoPE}(X)$. This difference can be decomposed into a linear component, $(w_\text{PE}-w_\text{NoPE})^\top x_q$, and a non-linear residual component, $\zeta(X) = \eta_{\text{PE}}-\eta_{\text{NoPE}}$. In our analysis, we focus on the dominant linear effect. This is justified because the residual term $\zeta(X)$ is a lower-order effect. As we argue in Appendix \ref{app:proof_of_pe_bounds}, the difference $\zeta(X)$ is bounded by a $t$-independent constant related to $\gamma_{\text{eff}}$. We therefore analyze the complexity of the linear function class: $\mathcal{F}_{\text{bias}} = \{x \mapsto (\Delta w)^\top x \mid \|\Delta w\|_2 \le C_{PE}\}$.

With this framework, we can get the result with PE.
\begin{theorem}[RC with PE]
\label{thm:rad_bound_with_pe}
With Eq.\eqref{eq:pe_effect} and Assumption \ref{assump:pe_effect}, let $\mathcal{F}_{S_t,\text{PE}}$ be the function class of a Transformer with a completely trainable PE. The Rademacher complexity for $\mathcal{F}_{S_t,\text{PE}}$ is bounded as:
\begin{equation}
\operatorname{Rad}_S(\mathcal{F}_{S_t,\text{PE}}) \lesssim O\left( \frac{L_f \sqrt{rD} \sqrt{\log t/r}}{\sqrt{m t}} \right) + \frac{C_{PE} \sqrt{d}}{\sqrt{m}}.
\end{equation}
\end{theorem}

The bound in Theorem \ref{thm:rad_bound_with_pe} can be interpreted as an upward shift of the No-PE complexity. $C_{PE}\sqrt{d}/\sqrt{m}$ represents the cost incurred by the PE module. Moreover, when $t$ becomes larger, $C_{PE} \sqrt{d}/\sqrt{m}$ dominates. This implies that with sufficiently long context, the model's generalization ability is no longer limited by the information in the context, but by the intrinsic property raised by the architectural complexity of the PE. The proof can be found in Appendix \ref{app:proof_of_pe_bounds}.

\subsection{ARC for Transformer}
\label{sec:arc_for_TF}
In this section, we extend our analysis by considering the adversarial robustness of the Transformer in ICL. The primary challenge in this setting is the $\text{max}$ operator defined in \ref{def:ARC}, which makes direct analysis difficult. To bypass this, we introduce the surrogate loss framework. We reveal that the adversarial attack increases the complexity bound by a factor $\Phi$, which quantifies the degree of the attack. Finally, we also extend this result with a trainable PE module.

\paragraph{The Surrogate Loss Framework.}
We follow a general method by defining a simpler surrogate loss that upper bounds the true adversarial loss, then analyzing the complexity of this surrogate.

\begin{proposition}[Surrogate Loss Upper Bound]
\label{prop:surrogate}
Assume the model function $f_\theta(X)$ is $L_x$-Lipschitz with respect to its input $X$. Then the true adversarial loss $\tilde{\ell}_\theta$ is upper-bounded by the surrogate loss $\ell_{\text{rob}}$:
\begin{equation}
\tilde{\ell}_\theta(X,y) \le \ell_{\text{rob}}(f_\theta(X), y), 
\end{equation}
where $\ell_{\text{rob}}(f_\theta(X), y) \triangleq \left( |f_\theta(X) - y| + L_x\varepsilon \right)^2$.
\end{proposition}

By the monotonicity of the Rademacher complexity, this proposition immediately implies that $\operatorname{Rad}_{S'}(\tilde{\mathcal{L}}_{\mathcal{F}_{S'_t}}) \le \operatorname{Rad}_{S'}(\mathcal{L}_{\text{rob}})$, where $\mathcal{L}_{\text{rob}}$ is the class of surrogate loss functions. The problem is now reduced to bounding the complexity of this surrogate class. The surrogate loss has a compositional structure $h(f(x), y)$, where $h(z, y) = (|z-y|+L_x\varepsilon)^2$. This structure allows us to apply Lemma \ref{lemma:Talagrand's contraction} to relate the complexity of the loss class back to the complexity of the original function class $\mathcal{F}_{S_t}$.
\begin{proposition}[Complexity Bound via Surrogate]
\label{prop:complexity_surrogate}
Let $\tilde{\mathcal{L}}_{\mathcal{F}_{S'_t}}$ be the adversarial loss class, where each element is $\tilde{\ell}_\theta(X,y)$ as defined in Eq.\eqref{eq:adv_loss}. Suppose the model's output is bounded by $|f_\theta(X)|\leq M$. Also, as shown in Lemma \ref{lemma:y_bound}, the response is bounded by $|y|\leq M_y$ with high probability. Then there exists a Lipschitz constant $L_h$ such that the Rademacher complexity of the adversarial loss class is bounded by the complexity of the standard function class:
\begin{equation}
\operatorname{Rad}_{S'}(\tilde{\mathcal{L}}_{\mathcal{F}_{S'_t}}) \le L_h \cdot \operatorname{Rad}_{S'}(\mathcal{F}_{S'_t}),
\end{equation}
where $L_h=2((M+M_y)+L_x\varepsilon)$.
\end{proposition}

\paragraph{The Impact of Attack on Solution Space.}
Under adversarial attack, the geometry of the solution space also changes. Following the same paradigm in Section \ref{sec:RC_for_TF}, as detailed below, we find that the diameter of the perturbed solution space can be expressed as the clean diameter scaled by a penalty term.

\begin{proposition}[Diameter of the Perturbed Solution Space]
\label{prop:perturbed_diameter}
The diameter of the adversarially perturbed solution space $W_{S'_t}$ is bounded by
\begin{equation}
\label{eq:perturbed_diameter}
\operatorname{Diam}(W_{S'_t})\lesssim O\left(\sqrt{r/t}\right) \cdot \Phi(\varepsilon, t, d),
\end{equation}
where $\Phi(\varepsilon, t, d) = \frac{1}{1-\sqrt{d/t}-\varepsilon/\sqrt{t}}$. This bound is meaningful only when $\varepsilon < \sqrt{t} - \sqrt{d}$.
\end{proposition}

Similar to Eq. \eqref{eq:diameter_contraction}, $\Phi$ arises from analyzing the geometry of the attacked solution space $W_{S'_t}$, but with a perturbation according to Lemma \ref{lemma:weyl_inequality}. As $\Phi>1$, the attack will enlarge the diameter of $W'_{S_t}$, which in turn enlarges the Rademacher complexity. When the attack strength $\varepsilon\rightarrow0$, the factor $\Phi$ approaches 1, and the bound in Eq. \eqref{eq:perturbed_diameter} recovers the clean bound from Eq. \eqref{eq:diameter_contraction}. Moreover, $\Phi$ is a decreasing function of $t$, which reveals that increasing the context length enhances adversarial robustness by reducing the attack's effect. The full proof for all the Propositions can be found in Appendix~\ref{app:proof_of_proposition}.

Then we extend the Lemma \ref{lemma:para_bound} to the adversarial setting as follows. In order to connect $W_{S'_t}$ and $\Theta_{S'_t}$, we make an assumption for the mapping.

\begin{assumption}[Robust Inverse Lipschitz Continuity]
\label{assump:robust_inverse_lipschitz}
Let $G_{\text{adv}}: \theta \mapsto w'_{\text{eff}}$, where $w'_{\text{eff}}$ is computed based on the perturbed context $S'_t$. We assume this mapping is also inverse Lipschitz continuous with a constant $L'_{G}$:
\begin{equation}
    \|\theta_1 - \theta_2\|_2 \le L'_{G} \|G_{\text{adv}}(\theta_1) - G_{\text{adv}}(\theta_2)\|_2.
\end{equation}
\end{assumption}
Assumption \ref{assump:robust_inverse_lipschitz} extends the stability property of the parameter-to-solution-space mapping $G$ to the adversarial setting. This is extended from Lemma \ref{lemma:para_bound} (justified in Lemma \ref{lemma:L_G_and_r}). We posit that a small, $\varepsilon$-bounded perturbation should only cause a correspondingly small, bounded degradation in the geometric properties of the mapping $G$. This ensures the inverse Lipschitz property is preserved, potentially with a slightly larger constant $L'_G$.

Given the above assumption, we finally present the adversarial RC result for Transformers without PE (Theorem \ref{thm:arc_without_pe}) and with PE (Theorem \ref{thm:arc_with_pe}):

\begin{theorem}[ARC without PE]
\label{thm:arc_without_pe}
Under Assumption \ref{assump:robust_inverse_lipschitz}, if $\varepsilon < \sqrt{t} - \sqrt{d}$, its Adversarial Rademacher complexity is bounded by:
\begin{equation}
\label{eq:bound_arc_nope}
\operatorname{Rad}_{S'}(\tilde{\mathcal{L}}_{\mathcal{F}_{S'_t}}) \lesssim O\left( \frac{L_h L_f \sqrt{rD} \sqrt{\log t/r}}{\sqrt{m t}} \right) \cdot \Phi(\varepsilon, t, d),
\end{equation}
where $\Phi(\varepsilon, t, d)$ is defined in Proposition~\ref{prop:perturbed_diameter}.
\end{theorem}
The proof combines the results from the preceding propositions. First, by Proposition~\ref{prop:complexity_surrogate}, we bound the ARC of the loss class by the Rademacher complexity of the function class, scaled by $L_h$. We then bound the complexity of the function class $\mathcal{F}_{S'_t}$ using the Dudley integral framework as in the standard case, but now applying it to the adversarially perturbed parameter space $\Theta_{S'_t}$. The decaying diameter of this space is a result of Proposition~\ref{prop:perturbed_diameter}, which introduces the $\Phi$ into the final bound. The full proof and the definition of $\Theta_{S'_t}$ is in Appendix~\ref{app:proof_of_arc_nope}.

Following the same decomposition method used for the standard RC bound, we can derive the ARC for the model with PE.

\begin{theorem}[ARC with PE]
\label{thm:arc_with_pe}
Let $\tilde{\mathcal{L}}_{\mathcal{F}_{S'_t,\text{PE}}}$ be the adversarial loss class for a Transformer with PE. Under Assumption \ref{assump:pe_effect} and \ref{assump:robust_inverse_lipschitz}, when $\varepsilon<\sqrt{t}-\sqrt{d}$, its Adversarial Rademacher complexity is bounded by:
\begin{align}
\operatorname{Rad}_{S'}(\tilde{\mathcal{L}}_{\mathcal{F}_{S'_t,\text{PE}}}) &\lesssim L_h \cdot \Bigg[ \left( \frac{C_{PE}\sqrt{d}}{\sqrt{m}} + \frac{C_{PE}\varepsilon}{\sqrt{m}} \right)  \\&+\left( O\left( \frac{L_f\sqrt{rD} \sqrt{\log t/r}}{\sqrt{m t}} \right)  \cdot \Phi(\varepsilon, t, d) \right)\nonumber \Bigg].
\end{align}
\end{theorem}
This bound reveals two distinct sources of adversarial attacks. Compared to Eq. \eqref{eq:bound_arc_nope}, $C_{PE}\varepsilon/{\sqrt{m}}$ implies that for stronger attacks, perturbing positional information itself may become a more significant threat than disturbing the data content. The proof of Theorem \ref{thm:arc_with_pe} can be found in Appendix~\ref{app:proof_of_arc_pe}.

With Theorem \ref{thm:rc_no_pe_final}, \ref{thm:rad_bound_with_pe}, \ref{thm:arc_without_pe} and \ref{thm:arc_with_pe}, we can formalize the impact of PE. Define the complexity gap as $\Delta = \operatorname{Rad}(\text{PE}) - \operatorname{Rad}(\text{noPE})$. In the clean setting, this gap is $\Delta_{\text{clean}} \approx C_{PE} \sqrt{d}/\sqrt{m}$. In the adversarial setting, the gap widens due to an additional attack-dependent term: $\Delta_{\text{adv}} \approx C_{PE} \sqrt{d}/\sqrt{m} + C_{PE} \varepsilon/\sqrt{m}$. Since $\Delta_{\text{adv}} = \Delta_{\text{clean}} + O(\varepsilon/\sqrt{m})$, it shows that the complexity cost of PE is magnified under attack, exposing Transformers with PE to be more vulnerable.

\section{Simulations}

In this section, we conduct simulations starting from training the transformer. We implemented the experiments based on \citet{garg2022can} with modifications.
\subsection{Setup}
Our data generation model follows Assumption~\ref{assump:data_model} with data dimension $d=5$. Following the experiment setting and sample size in \citet{trauger2024sequence}, we generate a total number of 10,300 tasks, which are split into a training set of 309 tasks and a validation set of 9,991 tasks. For the model architecture, we use a single-layer, single-head Transformer with a read-in layer and an embedding dimension of $d_m=1024$.

The training is based on \citet{garg2022can} but with a key modification to enable the measurement of a stable empirical risk. Instead of sampling new prompts at each training step, we create a \textbf{fixed training set}: for each of the 309 tasks in our training split, we pre-sample a single prompt of a fixed length $t$. The model is then trained for 3000 epochs to a stable training loss on this static dataset on a single NVIDIA A6000 GPU. Moreover, we set the batch size to the full training set size (309). This choice eliminates the stochasticity of mini-batch sampling, allowing for a cleaner measurement of the empirical risk and generalization gap. We use the Adam optimizer with a learning rate of $1 \times 10^{-5}$. The loss is the Mean Squared Error (MSE) computed only on the final query token of the prompt, $(\hat{y}_q - y_q)^2$. For the positional encoding, we scale the standard PE initialization by a factor of 25. This factor was chosen empirically to bring the initial norm of the PE into the same order of magnitude as the content embeddings, as in Eq.\eqref{eq:PE}. 

For the evaluation, we sample a large number of $t$-length prompts from the unseen tasks and evaluate the loss on the final query token.  The generalization gap is defined as the difference between the true risk on the validation set and the empirical risk on the fixed training set. 

To evaluate adversarial generalization, we use the Projected Gradient Descent (PGD) attack \citep{madry2017towards}. We maximize the MSE loss on the final query prediction by adding a perturbation to the prompt as in Eq. \eqref{eq:perturbed_input_format}. We constrain the perturbation using the $L_2$ norm, with a budget of $\varepsilon$. For our experiments, we use an attack configuration of $k=40$ iterations with a step size of $\alpha = \varepsilon/10$.

In the following experiments, each configuration is run 6 times with different random seeds, and we report the mean generalization gap. In alignment with our theory (Lemma~\ref{lemma:para_bound}), we focus on context lengths $t > d=5$.

\subsection{Standard Generalization Bounds}
\label{sec:simulation_PE}
In the experiment, we first validate our theoretical claim that incorporating a trainable PE module increases the model's complexity and thus its generalization gap. Figure~\ref{fig:gap_vs_t_standard} compares the generalization gap of the model with and without the PE module. 
\begin{figure}[ht]
    \centering
    \includegraphics[width=\columnwidth]{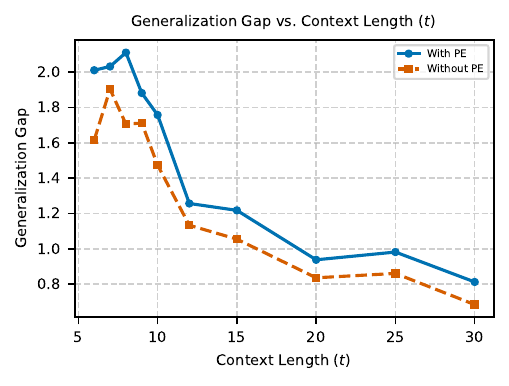}
    \caption{Mean Generalization Gap vs. Context Length ($t$) for models with and without Positional Encodings (PE).}
    \label{fig:gap_vs_t_standard}
\end{figure}
Similar to Theorem \ref{thm:rc_no_pe_final} and \ref{thm:rad_bound_with_pe}, in Figure \ref{fig:gap_vs_t_standard}, the model with a PE module consistently exhibits a larger generalization gap across all context lengths $t$. Furthermore, both curves demonstrate the $t$ decay trend in our analysis, showing that the model successfully utilizes more examples to improve the generalization bound.

In addition, we observe a transition phase for small context lengths (e.g., $t \leq 10$), where the gap has a large variation. When the number of examples $t$ is close to the dimension of the data $d$, the problem is less constrained, leading to a greater variance in the learned solution and thus a more unstable generalization gap. We noticed that this is also evidence of the "double descent" in modern machine learning \citep{belkin2019reconciling,nakkiran2021deep}. To be specific, in our framework, the Transformer implicitly learns an effective linear model of dimension $d$ as in Eq. \eqref{eq:w_eff}, which acts as the model complexity ($p$) in classic double descent. Then the error peaks at the interpolation threshold, where $p \approx n$, which is directly analogous to $d \approx t$ in our setting. Finally, as $t$ grows much larger than $d$, the system becomes over-determined and the gap decay becomes stable.

\subsection{Adversarial Generalization Bounds}

\paragraph{PE vs No-PE under Attack.}
We now extend the comparison to the adversarial setting. Figure~\ref{fig:attacked_gap_pe_vs_nope} shows the attacked generalization gap (for a fixed attack strength $\varepsilon=0.02$) for models with and without PE. 

\begin{figure}[ht]
    \centering
    \includegraphics[width=\columnwidth]{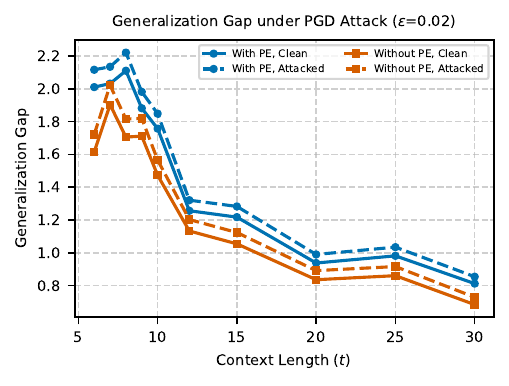}
    \caption{Mean Attacked Generalization Gap vs. Context Length ($t$) under a fixed PGD attack ($\varepsilon=0.02$).}
    \label{fig:attacked_gap_pe_vs_nope}
\end{figure}
This result is also consistent with our theoretical analysis as Theorem \ref{thm:arc_without_pe}. Our adversarial generalization bound is characterized by the term $\Phi(\varepsilon, t, d)$, which is an increasing function of the attack strength $\varepsilon$ under $\varepsilon \leq \sqrt{t}-\sqrt{d}$. For any $\varepsilon \ge 0$, the term $\Phi(\varepsilon, t, d)$ becomes greater than 1, thus amplifying the overall complexity bound. This directly explains why the generalization gap is systematically larger under adversarial attack compared to the clean setting. Moreover, following Theorem \ref{thm:arc_with_pe}, the inclusion of a trainable PE module explicitly increases the generalization gap, which is the same as the standard setting. 

\paragraph{Different Attack Strengths.}
Finally, we provide empirical evidence for our main theoretical contribution on ARC. Figure~\ref{fig:different_attack} plots the attacked generalization gap versus context length $t$, with each curve representing a different attack strength $\varepsilon$.
\begin{figure}[ht]
    \centering
    \includegraphics[width=\columnwidth]{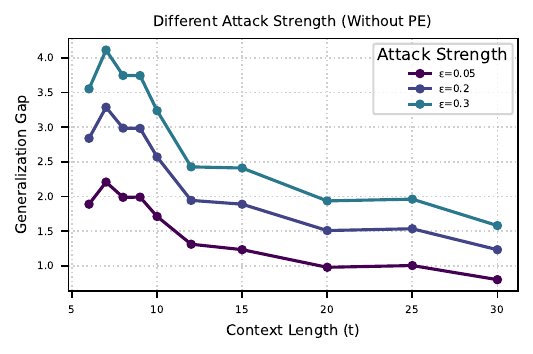}
    \caption{Mean Attacked Generalization Gap vs. Context Length ($t$) for different attack strengths ($\varepsilon$).}
    \label{fig:different_attack}
\end{figure}

For every attack strength $\varepsilon$ (e.g., 0.05, 0.2, 0.3), the generalization gap decreases as the context length increases. This demonstrates that the ICL mechanism remains effective in utilizing longer contexts to enhance generalization ability, even under adversarial attacks. Moreover, as $\varepsilon$ increases, the generalization gap also increases, which is aligned with Theorem \ref{thm:arc_without_pe}.

\subsection{Comparison with RoPE}
While our primary focus is on the impact of learnable parameters in PE, we also implement fixed encodings such as Rotary Position Encodings (RoPE) for comparison. Our theory suggests that since RoPE lacks learnable positional parameters, its complexity should be similar as the No-PE case. As detailed in Appendix~\ref{app:RoPE}, our experiments confirm this: RoPE shows a generalization gap comparable to the "No-PE" baseline, which is consistently lower than that of Trainable PE. This further validates that the increased generalization gap discussed in Section \ref{sec:simulation_PE} is indeed a consequence of the parameter complexity introduced by trainable PEs.

\section{Conclusion}
In this work, we provide a theoretical analysis of the impact of trainable Positional Encodings (PE) on the generalization and adversarial robustness of a one-layer Transformer performing ICL. Our analysis demonstrates that the increased model complexity from a trainable PE leads to a wider generalization gap. Furthermore, by deriving the ARC bounds for this setting, we quantified how this vulnerability is amplified under adversarial attack. Our theoretical findings were validated by empirical evidence.

There are several future directions. First, we only consider the simplest linear regression task. One direction is to extend the problem into other, more complicated tasks and even into real-world applications. Second, a natural extension is to consider understanding how positional encodings scale in multi-layer or deep Transformer architectures. Finally, with the complexity terms such as $L_f$ and $D$ identified in our bounds, one may develop new regularization techniques that directly penalize these terms to design more robust models.

\acknowledgments{
Weiyi He is supported by Open Philanthropy. Yue Xing is supported by NSF DMS 2515194, Open Philanthropy, NVIDIA Academic Grant Program and Google Cloud Research Credit.}

\bibliography{ref}

\section*{Checklist}
\begin{enumerate}

 \item For all models and algorithms presented, check if you include:
 \begin{enumerate}
   \item A clear description of the mathematical setting, assumptions, algorithm, and/or model. [Yes]
   \item An analysis of the properties and complexity (time, space, sample size) of any algorithm. [Yes]
   \item (Optional) Anonymized source code, with specification of all dependencies, including external libraries. [Yes]
 \end{enumerate}

 \item For any theoretical claim, check if you include:
 \begin{enumerate}
   \item Statements of the full set of assumptions of all theoretical results. [Yes]
   \item Complete proofs of all theoretical results. [Yes]
   \item Clear explanations of any assumptions. [Yes]     
 \end{enumerate}

 \item For all figures and tables that present empirical results, check if you include:
 \begin{enumerate}
   \item The code, data, and instructions needed to reproduce the main experimental results (either in the supplemental material or as a URL). [Yes]
   \item All the training details (e.g., data splits, hyperparameters, how they were chosen). [Yes]
         \item A clear definition of the specific measure or statistics and error bars (e.g., with respect to the random seed after running experiments multiple times). [Yes]
         \item A description of the computing infrastructure used. (e.g., type of GPUs, internal cluster, or cloud provider). [Yes]
 \end{enumerate}

 \item If you are using existing assets (e.g., code, data, models) or curating/releasing new assets, check if you include:
 \begin{enumerate}
   \item Citations of the creator If your work uses existing assets. [Not Applicable]
   \item The license information of the assets, if applicable. [Not Applicable]
   \item New assets either in the supplemental material or as a URL, if applicable. [Not Applicable]
   \item Information about consent from data providers/curators. [Not Applicable]
   \item Discussion of sensible content if applicable, e.g., personally identifiable information or offensive content. [Not Applicable]
 \end{enumerate}

 \item If you used crowdsourcing or conducted research with human subjects, check if you include:
 \begin{enumerate}
   \item The full text of instructions given to participants and screenshots. [Not Applicable]
   \item Descriptions of potential participant risks, with links to Institutional Review Board (IRB) approvals if applicable. [Not Applicable]
   \item The estimated hourly wage paid to participants and the total amount spent on participant compensation. [Not Applicable]
 \end{enumerate}

 \end{enumerate}

\appendix
\let\addcontentsline\oldaddcontentsline

\onecolumn

\section*{Appendix}

\setcounter{tocdepth}{2}   
\tableofcontents           
\clearpage

\section{Definitions}
\label{sec:appendix_definitions}

\begin{definition}[Covering Numbers]
Let $(\mathcal{F}, d)$ be a metric space. An $\varepsilon$-cover of a subset $A \subseteq \mathcal{F}$ is a set $\{c_1, \dots, c_N\}$ such that for every $f \in A$, there exists some $c_i$ with $d(f, c_i) \le \varepsilon$. The covering number, $N(A, d, \varepsilon)$, is the size of the smallest possible $\varepsilon$-cover. 
\end{definition}

\begin{definition}[Rademacher Complexity (RC)]
\label{def:RC}
Let $\mathcal{F}$ be a class of real-valued functions and $S = \{x_1, \dots, x_m\}$ be a fixed sample set. The empirical Rademacher complexity of $\mathcal{F}$ on $S$ is:
\begin{equation*}
\operatorname{Rad}_S(\mathcal{F}) \triangleq \frac{1}{m}\mathbb{E}_{\sigma} \left[ \sup_{f \in \mathcal{F}} \sum_{i=1}^m \sigma_i f(x_i) \right],
\end{equation*}
where $\sigma_i$ are i.i.d. and take values $\pm1$ each with half probability and $\sigma=(\sigma_1,\dots,\sigma_n)$.
\end{definition}

\begin{definition}[Adversarial Rademacher Complexity (ARC)]
\label{def:ARC}
Let $\mathcal{F}$ be a function class and $\ell$ be a loss function. The adversarial loss is the worst-case loss under a perturbation of budget $\varepsilon$:
\begin{equation*}
\tilde{\ell}_f(x, y) \triangleq \sup_{\|x'-x\| \le \varepsilon} \ell(f(x'), y).
\end{equation*}
The Adversarial Rademacher Complexity (ARC) is the Rademacher complexity of the corresponding adversarial loss class $\tilde{\mathcal{L}}_{\mathcal{F}} = \{ (x,y) \mapsto \tilde{\ell}_f(x, y) \mid f \in \mathcal{F} \}$:
\begin{equation*}
\operatorname{ARC}_S(\mathcal{F}) \triangleq \operatorname{Rad}_S(\tilde{\mathcal{L}}_{\mathcal{F}}) = \frac{1}{m}\mathbb{E}_{\sigma} \left[ \sup_{f \in \mathcal{F}} \sum_{i=1}^m \sigma_i \tilde{\ell}_f(x_i, y_i) \right].
\end{equation*}
\end{definition}

\begin{definition}
\label{def:norm}
Let $p,q,r,d \in \mathbb{N}$ and let $W \in \mathbb{R}^{r\times d}$. We define three standard matrix norms used throughout our analysis.

The Operator Norm (Induced Norm): The $L_q \to L_p$ operator norm, denoted $\|W\|_{q\rightarrow p}$, measures the maximum amplification factor that the matrix $W$ applies to any vector $x$. It is defined as the supremum:
$$\|W\|_{q\rightarrow p} \triangleq \sup_{x \neq 0} \frac{\|Wx\|_p}{\|x\|_q}.$$
A special case of this is the spectral norm, written as $\|W\|_2$, which corresponds to the $2\rightarrow 2$ norm and is equivalent to the largest singular value of $W$.

The Mixed $(q,p)$ Norm: The $\|W\|_{q,p}$ norm is calculated in a two-step process. First, the $\ell_q$-norm of each of the $d$ columns of $W$ is computed. Then, the $\ell_p$-norm is taken on the resulting $d$-dimensional vector of column norms:
$$\|W\|_{q,p} \triangleq \big\|[\|W_{:,1}\|_q,\dots, \|W_{:,d}\|_q]^\top\big\|_p.$$

The Frobenius Norm: The Frobenius norm, $\|W\|_F$, is the square root of the sum of all squared entries in the matrix, analogous to the Euclidean norm for vectors:
$$\|W\|_F \triangleq \sqrt{\sum_{i=1}^r \sum_{j=1}^d W_{ij}^2}.$$
This norm is also equivalent to the Euclidean norm of the matrix's singular values. Therefore, we have $\|W\|_2\leq \|W\|_F$ for any matrix $W$ .

\end{definition}

\newpage

\section{Useful Lemmas}
In this section, we list some useful lemmas, which are classic results and can be found in \citet{vershynin2018high}.
\begin{lemma}[Diameter of an Ellipsoid]
\label{lemma:DiameterEllipsoid}
For the ellipsoid $\{ w : (w - w_0)^\top A (w - w_0) \leq r \}$ where $A$ is symmetric positive definite, the diameter under the Euclidean norm is 
\begin{equation*}
2 \sqrt{r / \lambda_{\min}(A)},
\end{equation*}
where $\lambda_{\min}(A)$ is the smallest eigenvalue of $A$.
\end{lemma}

\begin{lemma}[Lower Bound for Minimum Singular Value]
\label{lemma:MinSingularValue}
Let $A$ be an $n \times d$ random matrix with entries i.i.d. $\mathcal{N}(0,1)$ and $n > d$. Then for every $k > 0$,
\begin{equation*}
\mathbb{P} \left( \sigma_{\min}(A) \leq \sqrt{n} - \sqrt{d} - k \right) \leq e^{-k^2 / 2}.
\end{equation*}
\end{lemma}
 
\begin{lemma}[Covering Number for Balls]
\label{lemma:covering number for balls}
Let $\Theta \subset V=\mathbb{R}^d$. Then
\begin{equation*}
\left( \frac{1}{\varepsilon} \right)^d \frac{\operatorname{vol(\Theta)}}{\operatorname{vol}(B)}\leq N(\Theta,||\cdot||_2,\varepsilon)\leq\left( \frac{3}{\varepsilon} \right)^d \frac{\operatorname{vol(\Theta)}}{\operatorname{vol}(B)},
\end{equation*}
where $B$ is the unit norm ball.
\end{lemma}

\begin{lemma}[Talagrand's Contraction Lemma (General Vector Form)]
\label{lemma:Talagrand's contraction}
Let $\mathcal{A} \subseteq \mathbb{R}^m$ be a set of vectors. For each $i=1, \dots, m$, let $\phi_i: \mathbb{R} \to \mathbb{R}$ be a $\rho$-Lipschitz function. Define the mapping $\Phi: \mathbb{R}^m \to \mathbb{R}^m$ which applies each function coordinate-wise:
\begin{equation*}
\Phi({a}) = (\phi_1(a_1), \phi_2(a_2), \dots, \phi_m(a_m)) \quad \text{for any} \quad {a} = (a_1, \dots, a_m) \in \mathcal{A}.
\end{equation*}
Let $\Phi \circ \mathcal{A} \triangleq \{ \Phi({a}) : {a} \in \mathcal{A} \}$ be the set of transformed vectors. Then the empirical Rademacher complexity of the transformed set is bounded by the complexity of the original set:
\begin{equation*}
\operatorname{Rad}_m(\Phi \circ \mathcal{A}) \le \rho \cdot \operatorname{Rad}_m(\mathcal{A}).
\end{equation*}
\end{lemma}
\begin{remark}
In many machine learning applications, all coordinate functions are the same, i.e., $\phi_1 = \phi_2 = \dots = \phi_m = \phi$. The lemma still holds with $\rho$ being the Lipschitz constant of the single function $\phi$.
\end{remark}

\begin{lemma}[Weyl Inequality]
\label{lemma:weyl_inequality}
For any matrix $A$ and $B$, we have the following Weyl inequality:
\begin{equation*}
    |\sigma_i(A+B)-\sigma_i(A)|\leq \|B\|_2,
\end{equation*}
where $\sigma_i$ is the singular value. A direct corollary is that
\begin{equation*}
\sigma_{\text{min}}(A+B) \ge \sigma_{\text{min}}(A) - ||B||_2,
\end{equation*}
where $||B||_2$ is the spectral norm and $\sigma_{\min}$ is the smallest singular value.
\end{lemma}

\newpage
\section{Proofs}
\subsection{Justification for Definition \ref{def:effective_param_space}}
\label{app:proof_of_def}
\begin{proof}
Under the condition of Theorem \ref{thm:rc_no_pe_final}, we have 
\begin{equation*}
    \|W_{QK}-W^*_{QK}(S_t)\|_F \leq \|W_{QK}-W^*_{QK}\|_F+\|W^*_{QK}-W^*_{QK}(S_t)\|_F\leq \gamma+\Delta_t.
\end{equation*}
Then we prove $\Delta_t\lesssim O(\sqrt{d/t})$ with high probability. 

Let $G_t := \frac{1}{t}\sum_{i=1}^t x_i x_i^\top$ with $x_i \stackrel{i.i.d.}{\sim}\mathcal N(0,I_d)$, and let
\begin{equation*}
W^*_{QK}(S_t) = \Psi(G_t), \qquad W^*_{QK} = \Psi(I_d).
\end{equation*}
We only require $\Psi$ to be locally Lipschitz in the spectral norm near $I_d$.

\begin{assumption}
\label{assump:psi_lipschitz}
There exists $L_\Psi>0$ and a neighborhood $\mathcal U$ of $I_d$ such that for all $G,G'\in\mathcal U$,
\begin{equation*}
\big\|\Psi(G)-\Psi(G')\big\|_F \le L_\Psi\,\big\|G-G'\big\|_2.
\end{equation*}
\end{assumption}

Assumption~\ref{assump:psi_lipschitz} is mild and holds for a broad class of smooth matrix functionals. It can be justified via Fr\'echet differentiability of $\Psi$ and a uniform bound on $\|\nabla\Psi\|$ in a neighborhood of $I_d$.

\begin{proposition}
\label{prop:delta_t_bound}
Under Assumption~\ref{assump:psi_lipschitz}, for any $\delta\in(0,1)$, with probability at least $1-\delta$, 
\begin{equation*}
\Delta_t := \big\|W^*_{QK}(S_t)-W^*_{QK}\big\|_F 
\le L_\Psi\,\big\|G_t-I_d\big\|_2
\le C\,L_\Psi\!\left(\sqrt{\frac{d}{t}}+\sqrt{\frac{\log(1/\delta)}{t}}\right),
\end{equation*}
where $C>0$ is a constant. In particular, $\Delta_t = O(\sqrt{d/t})$.
\end{proposition}

\begin{proof}
The first inequality follows immediately from Assumption \ref{assump:psi_lipschitz} by taking $G=G_t$ and $G'=I_d$:
\begin{equation*}
\Delta_t = \big\|\Psi(G_t)-\Psi(I_d)\big\|_F \le L_\Psi\,\big\|G_t-I_d\big\|_2.
\end{equation*}
For the second inequality, note that $tG_t\sim \text{Wishart}(d,t)$ with mean $I_d$. Standard spectral concentration for sample covariance matrices yields, for some absolute constant $C>0$ and any $\delta\in(0,1)$,
\begin{equation*}
\big\|G_t-I_d\big\|_2 \le C\left(\sqrt{\frac{d}{t}}+\sqrt{\frac{\log(1/\delta)}{t}}\right),
\end{equation*}
with probability at least $1-\delta$. Combining the two inequalities to get the results.
\end{proof}

Proposition \ref{prop:delta_t_bound} shows that the only term depending on the context length $t$ is factor $\sqrt{d/t}$ (up to a standard $\sqrt{\log(1/\delta)/t}$). All other quantities ($L_\Psi$, $C$) are constants and do not affect the decay profile in $t$. Hence, if we fix any $t_{\min}$ and work on the high-probability event, for all $t\ge t_{\min}$, we may set
\begin{equation*}
    \Delta_* := \sup_{t\ge t_{\min}} \Delta_t 
    = O\Big(L_\Psi\,\sqrt{d/t_{\min}}\Big),
\end{equation*}
and absorb $\Delta_*$ into $\gamma_{\text{eff}}=\gamma+\Delta_*$. This ensures that all subsequent constants remain $t$-independent.
\end{proof}

\subsection{Proof of Lemmas in Section \ref{sec:RC_for_TF}} 
\label{app:proof_of_lemmas}
In this section, we first prove the lemmas used in Section \ref{sec:RC_for_TF}, which are Lemma \ref{lemma:RC_dudley}, Lemma \ref{lemma:CoveringNumber_bound}, Lemma \ref{lemma:para_bound} and Lemma \ref{lemma:L_G_and_r}. 

\paragraph{Proof of Lemma \ref{lemma:RC_dudley}}
\begin{proof}
This bound follows from the classical Dudley entropy integral \citep{vershynin2018high}.
\end{proof}

\paragraph{Proof of Lemma \ref{lemma:CoveringNumber_bound}}
\label{app:proof_covering_bound}
\begin{proof}
We prove Lemma \ref{lemma:CoveringNumber_bound} in three steps: (i) a peeling reduction by
contraction, (ii) a Lipschitz link from parameters to the attention output, and (iii) covering
numbers for Euclidean balls and the translation from parameters to functions.

For the single-layer, single-head Transformer defined in Section~\ref{sec:TF_structure}, we analyze the parameter pair $\theta=(W_{\text{in}},W_{QK})\in\Theta_{S_t}$ defined in
Definition~\ref{def:effective_param_space}:
\[
\Theta_{S_t}=\Big\{(W_{\text{in}},W_{QK})\ :\ \|W_{QK}-W^*_{QK}(S_t)\|_F\le \gamma\Big\}.
\]
Other components $(W_V,W_c)$ are trained but uniformly bounded as in
Section~\ref{sec:TF_structure}: $\|W_c^\top\|_{1,\infty}\le B_{W_c}$ and $\|W_V^\top\|_{1,\infty}\le B_{W_V}$,
and $\sigma$ is $L_\sigma$-Lipschitz.

\paragraph{(i) Peeling via contraction.}
Apply Lemma \ref{lemma:Talagrand's contraction} layer by layer:
the composition of a bounded linear readout ($W_c$), a Lipschitz activation ($\sigma$),
and a bounded linear map ($W_V$) gives a global multiplicative constant
$B_{W_c}L_\sigma B_{W_V}$ relating the Rademacher complexity of the full function class and that of the attention block. Formally,
\begin{equation*}
\operatorname{Rad}_S(\mathcal{F}_{S_t})
\le B_{W_c}\,L_\sigma\,B_{W_V}\ \operatorname{Rad}_S(\mathcal U_{S_t}),   
\end{equation*}
where
\begin{equation*}
\mathcal U_{S_t}=\{ X\mapsto {h}_{\text{attn}}(X;W_{\text{in}},W_{QK})
:(W_{\text{in}},W_{QK})\in\Theta_{S_t}\},\qquad
{h}_{\text{attn}}=\text{RowSoftmax}\left(\frac{H W_{QK} H^\top}{\sqrt{d_m}}\right) H 
\end{equation*}
with $H=XW_{\text{in}}$.

\paragraph{(ii) Lipschitz link for the attention block.}
We show that the map $(W_{\text{in}},W_{QK})\mapsto {h}_{\text{attn}}$ is Lipschitz.
Let $\theta_1=(W_{\text{in},1},W_{QK,1})$ and $\theta_2=(W_{\text{in},2},W_{QK,2})$ in $\Theta_{S_t}$.
Set $H_i=XW_{\text{in},i}$ and
\begin{equation*}
M_i = \frac{1}{\sqrt{d_m}}H_i W_{QK,i} H_i^\top,\qquad
A_i = \operatorname{RowSoftmax}(M_i),\qquad i=1,2.
\end{equation*}
In the empirical norm, we have
\begin{equation*}
\|{h}_{\text{attn}}(\cdot;\theta_1)-{h}_{\text{attn}}(\cdot;\theta_2)\|_{m,2}
\;\le\; L_{\text{attn}}\ \|\theta_1-\theta_2\|_2,
\end{equation*}
with $L_{\text{attn}}$ depending on uniform bounds of $\|H\|_{2}$ and $\|W_{QK}\|_{2}$ over $\Theta_{S_t}$.
(These are controlled on a high-probability well-conditioned event $\mathcal G_t$ under the Gaussian
design; see Proof of Lemma \ref{lemma:para_bound}.) 

Combining with the peeling constant from step (i), we obtain the Lipschitz link
between functions and parameters claimed in the main text:
\begin{equation*}
\|f_{\theta_1}-f_{\theta_2}\|_{m,2}
\ \le\ B_{W_c}\,L_\sigma\,B_{W_V}\,L_{\text{attn}}\ \|\theta_1-\theta_2\|_2
\ \equiv\ L_f\ \|\theta_1-\theta_2\|_2.
\tag{C.1}
\label{eq:C1}
\end{equation*}

\paragraph{(iii) Covering numbers and translation.}
Let $D=d\cdot d_m + d_m^2$ be the ambient dimension of $\theta=(W_{\text{in}},W_{QK})$,
and $D_t=\operatorname{Diam}(\Theta_{S_t})$. Standard Euclidean ball covering from Lemma \ref{lemma:covering number for balls}:
gives
\begin{equation*}
\log N(\Theta_{S_t}, \|\cdot\|_2, \varepsilon) \leq D \log\left( \frac{3 D_t}{\varepsilon} \right).
\end{equation*}

Now \eqref{eq:C1} implies the covering translation:
if $\{\theta_k\}$ is an $(\varepsilon/L_f)$-cover of $\Theta_{S_t}$ in $\|\cdot\|_2$, then
$\{f_{\theta_k}\}$ is an $\varepsilon$-cover of $\mathcal F_{S_t}$ in $\|\cdot\|_{m,2}$, hence
\[
N\big(\mathcal{F}_{S_t},\|\cdot\|_{m,2},\varepsilon\big)
\ \le\ N\big(\Theta_{S_t},\|\cdot\|_2,\varepsilon/L_f\big).
\]
Similarly, taking the supremum over pairs gives
$\operatorname{Diam}(\mathcal{F}_{S_t})\le L_f\,\operatorname{Diam}(\Theta_{S_t})$, where $L_f=B_{W_c}\,L_\sigma\,B_{W_V}\,L_{\text{attn}}$.

\end{proof}

\paragraph{Proof of Lemma \ref{lemma:para_bound}}
\label{app:tail_event}
\begin{proof}

First, we justify the high probability "good event" $\mathcal{G}_t$ as mentioned in Lemma \ref{lemma:para_bound}. We define the tail event $\mathcal{T}_t$ as 
\begin{equation*}
\mathcal{T}_t := \left\{ S_t : \sigma_{\min}(X_t) \leq c_1 \sqrt{t} - c_2 \sqrt{d} \right\}.
\end{equation*}
Standard results from random matrix theory \citep{vershynin2018high} guarantee that for i.i.d. Gaussian data, the probability of this tail event decays exponentially with $t$:
\begin{equation*}
\mathbb{P}(\mathcal{T}_t) \leq 2d \exp(-c_3 t).
\end{equation*}
Our analysis is therefore conditioned on the complementary $\mathcal{G}_t = \mathcal{T}_t^c$. 

For any $\theta_1, \theta_2 \in \Theta_{S_t}$, let $w_1 = G(\theta_1)$ and $w_2 = G(\theta_2)$. By definition, $\|w_1 - w_2\|_2 \leq \operatorname{Diam}(W_{S_t})$. Applying the Eq. \eqref{eq:inv_lip}:
\begin{equation*}
\|\theta_1 - \theta_2\|_2 \leq L_G \|w_1 - w_2\|_2 \leq L_G \cdot \operatorname{Diam}(W_{S_t}).
\end{equation*}
Taking the supremum over all pairs $(\theta_1, \theta_2)$ yields $\operatorname{Diam}(\Theta_{S_t}) \leq L_G \cdot \operatorname{Diam}(W_{S_t})$.

Then applying Lemma \ref{lemma:MinSingularValue} to $X_t$ with $n = t$, we have for any $\delta > 0$:
\begin{equation*}
\mathbb{P} \left( \sigma_{\min}(X_t) \leq \sqrt{t} - \sqrt{d} - \sqrt{2 \log(1/\delta)} \right) \leq \delta.
\end{equation*}
Then there exists $K_1 > 0$ (depending on $d$ and $\delta$) such that with probability at least $1-\delta$:
\begin{equation*}
\sigma_{\min}(X_t) \geq K_1 \sqrt{t}.
\end{equation*}

By Lemma \ref{lemma:DiameterEllipsoid}, the diameter satisfies:
\begin{equation*}
\operatorname{Diam}(W_{S_t}) = \frac{2 \sqrt{r}}{\sigma_{\min}(X_t)} \leq \frac{2 \sqrt{r}}{K_1 \sqrt{t}}.
\end{equation*}
Thus, with high probability:
\begin{equation*}
\operatorname{Diam}(W_{S_t}) \lesssim O(\sqrt{r/t}).
\end{equation*}
Therefore, we now have the result that $\operatorname{Diam}(\Theta_{S_t})\lesssim L_G \operatorname{Diam}(W_{S_t}) \lesssim O\left(L_G\sqrt{r/t}\right).$
\end{proof}

\paragraph{Proof of Lemma \ref{lemma:L_G_and_r}}
\label{app:proof_of_LG_r}
\begin{proof}
Denote by $\hat w$ the OLS solution on $(X_t, y)$ as in Lemma \ref{lemma:para_bound}, and by
\begin{equation*}
\widehat{ y}_t^{\text{lin}}(\theta):=X_t\,w_{\text{eff}}(\theta)\ \in\mathbb{R}^t
\end{equation*}
the effective linear prediction vector induced by the Transformer parameters $\theta$ on the same inputs $\{x_i\}_{i=1}^t$.
Further, define the model prediction vector on the context by
\begin{equation*}
 g_t(\theta):=\big(g(\theta;S_t,x_1),\dots,g(\theta;S_t,x_t)\big)^\top,
\end{equation*}
where $g(\theta;S_t,x)$ is the scalar prediction produced by the Transformer when queried at $x$ under the prompt $S_t$.
Introduce the fitting error $ e_f:= y- g_t(\theta)$ and the nonlinearity residual
$\eta:= g_t(\theta)-\widehat{y}_t^{\text{lin}}(\theta)$ so that
\(
y-\widehat{ y}_t^{\text{lin}}(\theta)= e_f+\eta.
\)

Since $(w - \hat{w})^\top (X_t^\top X_t) (w - \hat{w}) \leq r$, we will first bound:
\begin{equation*}
r^*:=\big\|X_t\big(w_{\text{eff}}(\theta)-\hat w\big)\big\|_2^2.
\end{equation*}
By the OLS Pythagorean identity (orthogonal projection onto ${\rm span}(X_t)$), for any $w\in\mathbb{R}^d$ one has
\begin{equation*}
\| y - X_t w\|_2^2
=\| y - X_t \hat w\|_2^2\ +\|X_t(\hat w - w)\|_2^2.
\end{equation*}
Apply this with $w=w_\text{eff}(\theta)$ and rearrange to get
\begin{equation*}
\|X_t\big(w_{\text{eff}}(\theta)-\hat w\big)\|_2^2
=\| y - X_t w_{\text{eff}}(\theta)\|_2^2\ -\| y - X_t \hat w\|_2^2.
\end{equation*}
Since the OLS residual norm is nonnegative, $\| y - X_t \hat w\|_2^2\ge 0$, we deduce
\begin{equation*}
r^*=\|X_t\big(w_{\text{eff}}(\theta)-\hat w\big)\|_2^2\le\| y - X_t w_{\text{eff}}(\theta)\|_2^2.
\end{equation*}
By the definition of $ g_t(\theta)$ and $\widehat{ y}_t^{\text{lin}}(\theta)$, we have the exact decomposition
\begin{equation*}
 y - X_t w_{\text{eff}}(\theta) =\big( y- g_t(\theta)\big)+\big( g_t(\theta)-X_t w_{\text{eff}}(\theta)\big)=e_f+\eta.
\end{equation*}
This yields
\begin{equation*}
r^*\le\| e_f+\eta\|_2^2\leq2\| e_f\|_2^2+2\|\eta\|_2^2.
\end{equation*}
As shown in Proof \ref{app:proof_of_pe_bounds}, $\|\eta\|_2 \lesssim O(\gamma_{\text{eff}})$.
In particular, since the empirical fitting error is small, we can bound it as $\| e_f\|_2^2\le C_{\text{fit}}$, then
\(
r^*\le\ 2C_{\text{fit}}+2O(\gamma_{\text{eff}}):=C_r,
\)
a constant independent of $t$. Specifically, let $r=C_r$ to get the final result.

Fix $\theta_0\in K$ and set $y_0:=G(\theta_0)$. Since $J_G(\theta_0)$ has full row rank $d$, by the Constant Rank Theorem there exist neighborhoods 
$V\subseteq\Theta_{S_t}$ of $\theta_0$ and $U\subseteq\mathbb{R}^d$ of $y_0$, and $C^1$ diffeomorphisms
\[
\phi:V\to \phi(V)\subseteq\mathbb{R}^D,\qquad 
\psi:U\to \psi(U)\subseteq\mathbb{R}^d,
\]
such that in local coordinates the map $G$ takes the standard projection form
\[
(\psi\circ G\circ \phi^{-1})(x_1,\ldots,x_D) = (x_1,\ldots,x_d).
\]
Let $\iota:\mathbb{R}^d\to\mathbb{R}^D$ be the standard embedding 
$\iota(z_1,\ldots,z_d)=(z_1,\ldots,z_d,0,\ldots,0)$.
Define a local right inverse $s:U\to V$ by
\[
s(y)=\phi^{-1}\big(\iota(\psi(y))\big).
\]
Then $G(s(y))=y$ for all $y\in U$.

By the chain rule,
\[
J_s(y)= J_{\phi^{-1}}\big(\iota(\psi(y))\big)\cdot J_\iota\big(\psi(y)\big)\cdot J_\psi(y).
\]
We first note $\|J_\iota\|_2=1$. We claim that there exists a constant $M=M(c_G,c_U)$ such that on $V$ and $U$,
\[
\|J_\phi(\theta)\|_2,\ \|J_\phi(\theta)^{-1}\|_2,\ 
\|J_\psi(y)\|_2,\ \|J_\psi(y)^{-1}\|_2 \le M.
\]
Indeed, differentiating the normal form $\psi\circ G\circ \phi^{-1}=(x_1,\ldots,x_d)$ yields
\[
J_\psi\big(G(\theta)\big) J_G(\theta) J_{\phi^{-1}}(\phi(\theta)) = [\,I_d\ \ 0\,],
\]
hence
\[
J_G(\theta)= J_\psi\big(G(\theta)\big)^{-1}[\,I_d\ \ 0\,] J_\phi(\theta).
\]
Taking operator norms and using $\|[I_d\ 0]\|_2=1$, together with the uniform bounds 
$\|J_G(\theta)\|_2\le c_U$ and $\|J_G(\theta)^\dagger\|_2\le 1/c_G$ (since 
$\sigma_{\min}(J_G(\theta))\ge c_G$), standard perturbation arguments for products imply that
$J_\phi, J_\phi^{-1}, J_\psi, J_\psi^{-1}$ are uniformly bounded in operator norm on sufficiently small neighborhoods $V,U$, with bounds depending only on $(c_G,c_U)$. Thus the claimed $M(c_G,c_U)$ exists.

Therefore,
\[
\|J_s(y)\|_2 \le \big\|J_{\phi^{-1}}\big(\iota(\psi(y))\big)\big\|_2 \cdot \|J_\iota\|_2 \cdot \|J_\psi(y)\|_2
\le M\cdot 1\cdot M = M^2.
\]
Since $s$ is $C^1$ on $U$ and $\sup_{y\in U}\|J_s(y)\|_2\le M^2$, the mean value theorem implies
$s$ is Lipschitz on $U$ with $\operatorname{Lip}(s;U)\le M^2$. 
Set $L_G:=M^2$ to complete the proof.
\end{proof}

\subsection{Proof of Theorem \ref{thm:rc_no_pe_final}}
\label{app:proof_of_standard_rc}
\begin{proof}
Before we prove the final Theorem, we first state a technical lemma:
\begin{lemma}[Integral Bound]
\label{lemma:IntegralBound}
For $D_t \lesssim O(\sqrt{r/t})$, the core integral from Dudley's bound satisfies:
\begin{equation*}
\int_0^{D_t} \sqrt{ \log\left( \frac{K D_t}{\varepsilon} \right) } \, d\varepsilon \lesssim D_t \sqrt{\log(1/D_t)} \lesssim O\left( \frac{\sqrt{r\log t/r}}{\sqrt{t}} \right),
\end{equation*}
where $K$ is a constant.
\end{lemma}
\begin{proof}
We use the standard inequality $\sqrt{a+b} \le \sqrt{a} + \sqrt{b}$ for $a,b \ge 0$.
\begin{align*}
\int_0^{D_t} \sqrt{ \log(K) + \log(D_t/\varepsilon) } \, d\varepsilon &\le \int_0^{D_t} \left( \sqrt{\log(K)} + \sqrt{\log(D_t/\varepsilon)} \right) \, d\varepsilon \\
&= D_t\sqrt{\log(K)} + \int_0^{D_t} \sqrt{\log(D_t/\varepsilon)} \, d\varepsilon.
\end{align*}
The second term is a classic integral. By substituting $u = \varepsilon/D_t$, it becomes $D_t \int_0^1 \sqrt{\log(1/u)} \, du = D_t \cdot C$, where $C = \sqrt{\pi}/2$. A widely used, slightly looser bound that retains the dependency on $D_t$ is $O(D_t \sqrt{\log(1/D_t)})$. As this term dominates $D_t\sqrt{\log(K)}$ for small $D_t$, the bound holds. Substituting $D_t \propto \sqrt{r/t}$ yields $\log(1/D_t) \propto \log(\sqrt{t/r}) = \frac{1}{2}\log (t/r)$. The result follows.
\end{proof}

Then we start our proof from Eq. \eqref{eq:start_point} and first apply inequalities \eqref{eq:covering_translation}.
\begin{align*}
\operatorname{Rad}_S(\mathcal{F}_{S_t}) 
&\lesssim \frac{1}{\sqrt{m}} \int_0^{\operatorname{Diam}(\mathcal{F}_{S_t})} \sqrt{\log N(\mathcal{F}_{S_t}, \|\cdot\|_{m,2}, \varepsilon)} \, d\varepsilon \\
&\le \frac{1}{\sqrt{m}} \int_0^{L_f D_t} \sqrt{\log N(\Theta_{S_t}, \|\cdot\|_2, \varepsilon / L_f)} \, d\varepsilon
\end{align*}
Now, we perform a change of variables in the integral. Let $u = \varepsilon / L_f$. This implies $\varepsilon = u \cdot L_f$ and $d\varepsilon = L_f \, du$. The integration limits for $u$ become $0$ to $D_t$.
\begin{align*}
\operatorname{Rad}_S(\mathcal{F}_{S_t}) 
&\lesssim \frac{1}{\sqrt{m}} \int_0^{D_t} \sqrt{\log N(\Theta_{S_t}, \|\cdot\|_2, u)} \cdot (L_f \, du) \\
&= \frac{L_f}{\sqrt{m}} \int_0^{D_t} \sqrt{\log N(\Theta_{S_t}, \|\cdot\|_2, u)} \, du \\
&\lesssim \frac{L_f}{\sqrt{m}} \int_0^{D_t} \sqrt{D \log\left( \frac{3 D_t}{\varepsilon} \right)} \, d\varepsilon \quad (\text{by Eq. \eqref{eq:para_bound}})\\
&= \frac{L_f \sqrt{D}}{\sqrt{m}} \int_0^{D_t} \sqrt{ \log\left( \frac{3 D_t}{\varepsilon} \right)} \, d\varepsilon \\
&\lesssim \frac{L_f \sqrt{rD}}{\sqrt{m}} \cdot O\left( \frac{\sqrt{\log t/r}}{\sqrt{t}} \right) \quad (\text{by Lemma \ref{lemma:IntegralBound}}),
\end{align*}
which derives Theorem \ref{thm:rc_no_pe_final}.
\end{proof}

\subsection{Proof of Theorem \ref{thm:rad_bound_with_pe}}
\label{app:proof_of_pe_bounds}
\begin{proof}

Following the architecture of the Transformer, define 
\begin{equation*}
   f_{\text{bias}}(X)=\Delta w^\top x_q + \zeta(X),
\end{equation*}
where $\Delta w:=w_{\text{PE}}-w_{\text{NoPE}}$ and $\zeta(X)$ is the non-linear error. From Assumption \ref{assump:pe_effect}, $\|\Delta w\|_2 \leq C_{\text{PE}}$. We will leave the discussion for $\zeta(X)$ to the end of this section and focus on the main linear effect first. 

For Theorem \ref{thm:rad_bound_with_pe}, we only need to bound the Rademacher Complexity for $\mathcal{F}_{bias}=\{X \mapsto (\Delta w)^\top x_q \mid \|\Delta w\|_2 \leq C_{PE} \}$. Following the definition of Rademacher Complexity and on sample $S=\{X^{j}\}^m_{j=1}$, we have:
\begin{align*}
    \operatorname{Rad}_S(\mathcal{F}_{\text{bias}})&=\frac{1}{m} \mathbb{E}_\sigma \Bigg[\sup_{\|\Delta w\|_2\leq C_{PE}} \sum_{j=1}^m \sigma_j (\Delta w)^\top x_q^{(j)}\Bigg] \\
    &= \frac{1}{m} \mathbb{E}_\sigma \Bigg[\sup_{\|\Delta w\|_2\leq C_{PE}} (\Delta w)^\top \left(\sum_{j=1}^m \sigma_j x_q^{(j)} \right)\Bigg] \\
    &\leq \frac{C_{PE}}{m} \mathbb{E}_{\sigma} \big\|\sum_{j=1}^m \sigma_j x_q^{(j)} \big\|_2  \\
    &\leq \frac{C_{PE}}{m} \sqrt{\mathbb{E}_{\sigma} \big\|\sum_{j=1}^m \sigma_j x_q^{(j)} \big\|^2_2} = \frac{C_{PE}}{m} \sqrt{\big\| \sum_{j=1}^m x_q^{(j)}\big\|^2_2}.
\end{align*}
For $\big\|\sum_{j=1}^m x_q^{(j)}\big\|^2_2$, we know from the concentration inequality that with high probability $1-\delta$, 
\begin{equation*}
    \big\|\sum_{j=1}^m x_q^{(j)}\big\|^2_2 \leq \mathbb{E}\big\|\sum_{j=1}^m x_q^{(j)}\big\|^2_2=md+O(\delta).
\end{equation*} 
For simplicity, we take the result that 
\begin{equation*}
    \operatorname{Rad}_S(\mathcal{F}_{\text{bias}})\lesssim \frac{C_{PE} \sqrt{d}}{\sqrt{m}},
\end{equation*}
which gives the result.

Then we show that $\zeta(X)$ is bounded by $O(\gamma_{\text{eff}})$. According to the Definition \ref{def:effective_param_space}, let $\|W_{\text{in}}\|\leq B_{\text{in}}$. If for any $\theta \in \Theta_{S_t}$, we let $\theta=\theta^*+\Delta\theta$, where $\theta^*$ is defined in Lemma \ref{lemma:L_G_and_r}, then
\begin{equation*}
    \|\Delta\theta\|_2 \leq \|W_{\text{in}}-W_{\text{in}}^*\|_F + \|W_{QK}-W^*_{QK}(S_t)\|_F \leq O(\gamma_{\text{eff}}),
\end{equation*}
with $B_{\text{in}}$ absorbed in $O(\gamma_{\text{eff}})$.
Then from Lemma \ref{lemma:CoveringNumber_bound}, we have
\begin{equation*}
    \|f_\theta-f_\theta^*\|_{m,2} \leq L_f \|\theta-\theta^*\|_2=L_f\|\Delta\theta\|_2 \lesssim O(\gamma_{\text{eff}}).
\end{equation*}
Also, for the effective linear weight,
\begin{equation*}
    \|w^\theta_{\text{eff}}-w^{\theta^*}_{\text{eff}}\|_2\leq L_w \|\theta-\theta^*\|_2\lesssim O(\gamma_{\text{eff}}).
\end{equation*}
Let $\zeta(X)=\eta_{\text{PE}}(X)-\eta_{\text{NoPE}}(X)$ and $\eta_\theta(X)=f_\theta(X)-(w^\theta_{\text{eff}})^\top x_q$, then we have decomposition
\begin{equation*}
    \eta_\theta(X)=f_\theta(X)-f_{\theta^*}(X)-(w^\theta_{\text{eff}}-w^{\theta^*}_{\text{eff}})^\top x_q + [f_{\theta^*}(X)-(w^{\theta^*}_{\text{eff}})^\top x_q],
\end{equation*}
while the third term is 0 according to the definition of $\theta^*$.

Hence, taking the empirical norm, 
\begin{equation*}
    \|\eta_\theta(X)\|_{m,2} \leq \|f_\theta-f_\theta^*\|_{m,2}+\|w^\theta_{\text{eff}}-w^{\theta^*}_{\text{eff}}\|_2\cdot \|x_q\|_{m,2} \lesssim O(\gamma_{\text{eff}}).
\end{equation*}
Finally,
\begin{equation*}
    \|\zeta(X)\|_{m,2}=\|\eta_{\text{PE}}(X)-\eta_{\text{NoPE}}(X)\|_{m,2} \leq \|\eta_{\text{PE}}(X)\|_{m,2} + \|\eta_{\text{NoPE}}(X)\|_{m,2} \leq O(\gamma_{\text{eff}}).
\end{equation*}
This completes the proof. Therefore, we can focus on the main linear effect in the following.
\end{proof}

\subsection{Proof of Propositions in Section \ref{sec:rc_with_pe}}
\label{app:proof_of_proposition}
In this section, we first prove the propositions in Section \ref{sec:rc_with_pe}, which are Proposition \ref{prop:surrogate}, Proposition \ref{prop:complexity_surrogate}, and Proposition \ref{prop:perturbed_diameter}.

\paragraph{Proof of Proposition \ref{prop:surrogate}}
\label{app:proof_surrogate}

\begin{proof}
For any perturbation $X'$ such that $\|X'-X\|_F \le \varepsilon$, we use the triangle inequality on $|f_\theta(X') - y|$:
\begin{align*}
    |f_\theta(X') - y| &= |(f_\theta(X') - f_\theta(X)) + (f_\theta(X) - y)| \\
    &\le |f_\theta(X') - f_\theta(X)| + |f_\theta(X) - y| \\
    &\le L_x \|X'-X\|_F + |f_\theta(X) - y| \\
    &\le L_x\varepsilon + |f_\theta(X) - y|.
\end{align*}
By squaring it we can get:
\begin{equation*}
\sup_{\|X'-X\| \le \varepsilon} (f_\theta(X') - y)^2 \le \left( |f_\theta(X) - y| + L_x\varepsilon \right)^2. 
\end{equation*}
Thus, $\tilde{l}_\theta(X,y) \le \ell_{\text{rob}}(f_\theta(X), y)$.
\end{proof}
\paragraph{Proof of Proposition \ref{prop:complexity_surrogate}}
\begin{proof}
The Rademacher complexity is monotonic. Since every function in $\tilde{\mathcal{L}}_{\mathcal{F}_{S'_t}}$ is upper-bounded by a corresponding function in $\mathcal{L}_{\text{rob}}$, we have:
\begin{equation*} 
\operatorname{Rad}_{S'}(\mathcal{L}_{\tilde{\mathcal{L}}_{\mathcal{F}}}) \le \operatorname{Rad}_{S'}(\mathcal{L}_{\text{rob}}). 
\end{equation*}
The surrogate loss class $\mathcal{L}_{\text{rob}}$ has the structure $\{ (x,y) \mapsto h(f(x),y) \mid f \in \mathcal{F}\}$, where $h(z,y) = (|z-y|+L_f\varepsilon)^2$. Assuming bounded outputs, $h$ is $L_h$-Lipschitz with respect to its first argument $z$. By Lemma \ref{lemma:Talagrand's contraction}:
\begin{equation*}
\operatorname{Rad}_{S'}(\mathcal{L}_{\text{rob}}) = \operatorname{Rad}_{S'}(h \circ \mathcal{F}) \le L_h \cdot \operatorname{Rad}_{S'}(\mathcal{F}). 
\end{equation*}
Combining the inequalities gives the desired result.
For $L_h$, follow that $L_h=\sup_z |h'(z)|$, we then need to find $|h'(z)|$. Easily to know that $|h'(z)|=2(|z-y|+L_x\varepsilon)$, which can then be used for the final $L_h$, that is $L_h=2((M+M_y)+L_x\varepsilon)$.

Finally, we provide evidence that the response $y$ in Proposition \ref{prop:complexity_surrogate} is bounded with high probability.
\begin{lemma}[High-Probability Bound on Responses]
\label{lemma:y_bound}
Under the data generation model in Assumption~\ref{assump:data_model}, where $\mu \sim \mathcal{N}(0, I_d/d)$ and $y = \mu^\top x$ with $x \sim \mathcal{N}(0, I_d)$, for any failure probability $\delta_y \in (0, 1)$, there exists a constant $M_y$ such that:
\begin{equation*}
    \mathbb{P}_{w,x}(|y| > M_y) \le \delta_y.
\end{equation*}
\end{lemma}

\begin{proof}
The response $y$ has two sources of randomness: the true weight vector $\mu$ and the input $x$. Our proof proceeds in two steps.

Since $\mu \sim \mathcal{N}(0, I_d/d)$, each component $\mu_i \sim \mathcal{N}(0, 1/d)$. Thus, $d \cdot \|\mu\|_2^2 = \sum_{i=1}^d (\mu_i\sqrt{d})^2$ follows a Chi-squared distribution with $d$ degrees of freedom, $\chi^2_d$. A standard concentration inequality for Chi-squared variables states that for any $\delta_\mu \in (0,1)$, $\|\mu\|_2^2 \leq (1 + \sqrt{2\log(1/\delta_\mu)/d})^2$ with probability at least $1-\delta_\mu$ over the draw of $\mu$. For simplicity, this means we can find a constant $C_\mu$ (e.g., $C_\mu=2$ for large $d$) such that $\|\mu\|_2^2 \le C_\mu$ with overwhelmingly high probability.

We now condition on the high-probability event that $\|\mu\|_2^2 \le C_\mu$. Conditioned on a fixed $\mu$ satisfying this, $y = \mu^\top x$ is a zero-mean Gaussian random variable (as it is a linear combination of Gaussians). Its variance is:
\begin{equation*}
\operatorname{Var}(y | \mu) = \operatorname{Var}(\mu^\top x | \mu) = \|\mu\|_2^2 \triangleq \sigma_y^2. 
\end{equation*}
Since $\|\mu\|_2^2 \le C_\mu$, we have $\sigma_y^2 \le C_\mu$. Let $Z = y/\sigma_y \sim \mathcal{N}(0, 1)$. The standard tail bound gives 
\begin{equation*}
\mathbb{P}(|y| > k\sigma_y) \le 2e^{-k^2/2}.
\end{equation*}
To make this probability less than a desired $\delta_y'$, we can choose $k = \sqrt{2 \log(2/\delta_y')}$. This gives us a bound on $y$:
\begin{equation*}
M_y = k\sigma_y \le \sqrt{C_\mu} \sqrt{2 \log(2/\delta_y')}. 
\end{equation*}
By a union bound, the probability that either $\|\mu\|_2^2 > C_\mu$ or $|y| > M_y$ is at most $\delta_\mu + \delta_y'$. By setting these failure probabilities appropriately, we can ensure the total failure probability is less than any given $\delta_y$. Thus, the statement holds with high probability over the joint distribution of $\mu$ and $x$.
\end{proof}

\end{proof}
\paragraph{Proof of Proposition \ref{prop:perturbed_diameter}}
\begin{proof}
    Let $X'_t=X_t + \Delta_X$, then following Wely Inequality in Lemma \ref{lemma:weyl_inequality}, we have
    \begin{equation*}
        \sigma_{\text{min}}(X'_t)=\sigma_{\text{min}}(X_t+\Delta_X) \ge \sigma_{\text{min}}(X_t)-\|\Delta_X\|_2.
    \end{equation*}

From the definition in Eq. \eqref{eq:adv_loss}, that is $\|X'-X\|_F\leq \varepsilon$, we will get $|\Delta_X\|_2\leq\|\Delta_X\|_F \leq \varepsilon$, then
    \begin{equation*}
        \sigma_{\text{min}}(X'_t)=\sigma_{\text{min}}(X_t+\Delta_X) \ge \sigma_{\text{min}}(X_t)-\varepsilon \ge \sqrt{t}-\sqrt{d}-\varepsilon.
    \end{equation*}
and then 
\begin{equation*}
    \operatorname{Diam}(W'_{S_t})\lesssim \frac{2\sqrt{r}}{\sqrt{t}-\sqrt{d}-\varepsilon}=\frac{2\sqrt{r}}{\sqrt{t}} \cdot \frac{\sqrt{t}}{\sqrt{t}-\sqrt{d}-\varepsilon}.
\end{equation*}
Therefore, we have the  definition for $\Phi$, that is 
\begin{equation*}
\Phi(\varepsilon,t,d)=\frac{\sqrt{t}}{\sqrt{t}-\sqrt{d}-\varepsilon}=\frac{1}{1-\sqrt{d/t}-\varepsilon/\sqrt{t}}.    
\end{equation*}
\end{proof}
\subsection{Proof of Theorem \ref{thm:arc_without_pe}}
\label{app:proof_of_arc_nope}
\begin{proof}
From Proposition \ref{prop:complexity_surrogate}, we have
\begin{equation*}
    \operatorname{Rad}_{S'}(\tilde{\mathcal{L}}_{\mathcal{F}_{S'_t}}) \le L_h \cdot \operatorname{Rad}_{S'}(\mathcal{F}_{S'_t}).
\end{equation*}
Then we follow the proof of Theorem \ref{thm:rc_no_pe_final} (\ref{app:proof_of_standard_rc}), but with $D'_t=\operatorname{Diam}(\Theta'_{S_t})$. $\Theta'_{S_t}$ is defined similarly to $\Theta_{S_t}$ in Definition \ref{def:effective_param_space} with $S'_t$:
\begin{equation*}
\Theta_{S'_t} := \Big\{ (W_{\text{in}}, W_{QK}) : \|W_{QK} - W^*_{QK}(S'_t)\|_F \le \gamma'_{\text{eff}} \Big\},
\end{equation*}
Let $\Delta'_t := \|W^*_{QK}(S'_t)-W^*_{QK}\|_F$. Following Appendix~\ref{app:proof_of_def}, to show that $\Delta'_t \lesssim O(\sqrt{d/t})$ with high probability, since
\begin{equation*}
    \|G'_t-I_d\|_2\leq \|G_t-I_d\|_2+\|G'_t-G_t\|_2,
\end{equation*}
we only need to consider $\|G'_t-G_t\|_2$, where $G'_t=\frac{1}{t} \sum_{i=1}^t  x'_i x_i^{'\top}$. Let $\delta_i=x'_i-x_i$, then
\begin{align*}
    \|G'_t-G_t\|_2&=\|\frac{1}{t} \sum_{i=1}^t (x_i\delta_i^\top + \delta_i x_i^\top +\delta_i \delta_i^\top) \|_2 \\
    &\leq \frac{1}{t} \sum_{i=1}^t \|\delta_i\|_2^2 + \frac{2}{t}\sum_{i=1}^t \|\delta_i\|_2\|x_i\|_2 \\
    &\leq \frac{\varepsilon^2}{t}+\frac{2}{t}\left(\sum_{i=1}^t\|\delta_i\|_2^2\right)^{\frac{1}{2}}\left(\sum_{i=1}^t\|x_i\|_2^2\right)^{\frac{1}{2}} \\
    &\leq \frac{\varepsilon^2}{t}+\frac{2\sqrt{d}}{\sqrt{t}}\varepsilon.
\end{align*}
This gives the result. Therefore similarly, for any fixed $t_{\min}$ and all $t \ge t_{\min}$, we can set constant
$\gamma'_{\text{eff}} = \gamma + \Delta^{'*}$ and $\Delta^{'*} := \sup_{t \ge t_{\min}} \Delta'_t$.

Then it is easy to know
\begin{align*}
    \operatorname{Rad}_{S'}(\mathcal{F}_{S'_t}) 
&\lesssim \frac{L_f}{\sqrt{m}} \int_0^{D'_t} \sqrt{\log N(\Theta'_{S_t}, \|\cdot\|_2, u)} \, du \\
&\lesssim \frac{L_f}{\sqrt{m}} \int_0^{D'_t} \sqrt{D \log\left( \frac{3 D'_t}{\varepsilon} \right)} \, d\varepsilon \\
&\lesssim \frac{L_f \sqrt{rD}}{\sqrt{m}} \cdot \Phi \cdot O\left( \frac{\sqrt{\log t/r}}{\sqrt{t}} \right),
\end{align*}
which completes the proof.
\end{proof}

\subsection{Proof of Theorem \ref{thm:arc_with_pe}}
\label{app:proof_of_arc_pe}
\begin{proof}
    From Proposition \ref{prop:complexity_surrogate}, we have
        \begin{equation*}
            \operatorname{Rad}_{S'}(\tilde{\mathcal{L}}_{\mathcal{F}_{S'_t,\text{PE}}}) \le L_h \cdot \operatorname{Rad}_{S'}(\mathcal{F}_{S'_t,\text{PE}}).
        \end{equation*}
    Following the function class decomposition in Eq. \eqref{eq:subadditivity}, we have
        \begin{equation*}
            \operatorname{Rad}_{S'}(\mathcal{F}_{S'_t,\text{PE}}) \le \operatorname{Rad}_{S'}(\mathcal{F}_{S'_t}) + \operatorname{Rad}_{S'}(\mathcal{F}_{\text{bias}}).
        \end{equation*}
    Note that the bias function itself does not change under adversarial attack, it is only determined by the presence of PE. What changes is only the evaluation set (from $S_t$ to $S'_t$), which affects the empirical Rademacher complexity.
    Similarly, we only need to bound the bias term $\operatorname{Rad}_{S'}(\mathcal{F}_{\text{bias}})$, and we consider
    \begin{equation*}
        f_{\text{bias}}(X')=\Delta w^\top x'_q +\zeta(X').
    \end{equation*}
    Follow the same process in proof \ref{app:proof_of_pe_bounds}, for the attacked sample $S'=\{X'^{(j)}\}_{j=1}^m$, we have its empirical Rademacher complexity as  
    \begin{equation*}
        \operatorname{Rad}_{S'}(\mathcal{F}_{\text{bias}})=\frac{1}{m} \mathbb{E}_\sigma \Bigg[\sup_{\|\Delta w\|_2\leq C_{PE}} \sum_{j=1}^m \sigma_j (\Delta w)^\top x_q'^{(j)}\Bigg]
         \leq \frac{C_{PE}}{m} \mathbb{E}_{\sigma} \big\|\sum_{j=1}^m \sigma_j x_q'^{(j)} \big\|_2. 
    \end{equation*}
    Let $\delta_q^{(j)}=x_q'^{(j)}-x_q^{(j)}$, therefore 
    \begin{equation*}
        \|\delta_q^{(j)}\|_2=\|x_q'^{(j)}-x_q^{(j)}\|_2 \leq \|X'-X\|_2 \leq \|X'-X\|_F \leq \varepsilon.
    \end{equation*}
    Then we have
    \begin{equation*}
        \big\|\sum_{j=1}^m \sigma_j x_q'^{(j)} \big\|_2\leq \big\|\sum_{j=1}^m \sigma_j x_q^{(j)} \big\|_2 + \big\|\sum_{j=1}^m \sigma_j \delta_q'^{(j)} \big\|_2.
    \end{equation*}
    Since 
    \begin{equation*}
        \mathbb{E}_{\sigma}\big\|\sum_{j=1}^m \sigma_j \delta_q'^{(j)} \big\|_2 \leq \left(\sum_{j=1}^m \|\delta_q^{(j)}\|_2^2\right)^\frac{1}{2}\leq \sqrt{m}\varepsilon,
    \end{equation*}
    we get the result:
    \begin{equation*}
        \operatorname{Rad}_{S'}(\mathcal{F}_{\text{bias}}) \lesssim \frac{C_{\text{PE}}\sqrt{d}}{\sqrt{m}}+\frac{C_{\text{PE}}\varepsilon}{\sqrt{m}}.
    \end{equation*}
    Finally
    \begin{align*}
        \operatorname{Rad}_{S'}(\tilde{\mathcal{L}}_{\mathcal{F}_{S'_t,\text{PE}}}) &\lesssim L_h \cdot \Bigg[ \left( \frac{C_{PE}\sqrt{d}}{\sqrt{m}} + \frac{C_{PE}\varepsilon}{\sqrt{m}} \right)  \\&+\left( O\left( \frac{L_f\sqrt{rD} \sqrt{\log t/r}}{\sqrt{m t}} \right)  \cdot \Phi(\varepsilon, t, d) \right)\nonumber \Bigg].
\end{align*}
\end{proof}

\newpage
\section{Extension to sub-Gaussian designs}
\label{app:extention_sub_gaussian}
In this section, we claim that our proof can be extended to sub-Gaussian distributions. Importantly, all dependence on the 
context length $t$ and sample size $m$ remains unchanged; only constants are modified. We first provide the definition for sub-Gaussian distribution.
\begin{definition}
A zero-mean random variable $X$ is sub-Gaussian if there exists $K>0$ such that
\[
\mathbb{P}(|X| > t) \le 2\exp\!\left( - \frac{t^2}{K^2} \right)
\quad \forall t>0 ,
\]
and its sub-Gaussian norm is denoted $\|X\|_{\psi_2} = K$.  
\end{definition}
This class includes Bernoulli variables, bounded variables (e.g.\ normalized images, token embeddings), mixtures of Gaussians, and many real-world data-generating processes. Thus extending our analysis to sub-Gaussian designs significantly broadens its applicability. We extend Assumption~\ref{assump:data_model} as follows.

\begin{assumption}[sub-Gaussian design]
\label{assump:subgaussian}
The input vectors $\{x_i\}_{i=1}^t$ are i.i.d., zero-mean, $\mathbb{E}[x_i x_i^\top] = I_d$, and sub-Gaussian with $\|x_i\|_{\psi_2} \le K$ for some constant $K$.
\end{assumption}

\subsection{Minimum singular value of the data matrix}

Lemma~\ref{lemma:para_bound} relies on the lower bound 
$\sigma_{\min}(X_t) \gtrsim \sqrt{t} - \sqrt{d}$ as shown in \ref{app:tail_event}. For sub-Gaussian designs, a similar bound holds with modified constants.

\begin{lemma}[Minimum singular value under sub-Gaussian design]
\label{lemma:subgaussian_smin}
Under Assumption~\ref{assump:subgaussian}, there exist positive constants $c_1(K), c_2(K)$ depending only on the sub-Gaussian norm $K$ such that with probability at least $1 - 2\exp(-c_2 t)$,
\[
\sigma_{\min}(X_t) 
\;\ge\; 
c_1(K)\big( \sqrt{t} - \sqrt{d} \big).
\]
\end{lemma}

\noindent
This is a direct result of standard sub-Gaussian random matrix theory \citep{vershynin2018high}.  
Therefore the term $1/{\sigma_{\min}(X_t)}$ used in our generalization proofs preserves its $1/\sqrt{t}$ dependence exactly.

\subsection{Concentration of quadratic norms}
The only place where the Gaussian assumption enters the proof of 
Theorem \ref{thm:rad_bound_with_pe} is the concentration of
\(
\left\|\sum_{j=1}^m \sigma_j x_q^{(j)}\right\|_2^2
\)
(or equivalently $\|\sum x^{(j)}\|_2^2$ for simplicity, see Appendix~\ref{app:proof_of_pe_bounds}). The results in Appendix~\ref{app:proof_of_pe_bounds} also hold for Sub-Gaussian vectors, introducing only an additional multiplicative constant related to $K$ into the complexity term $C_{\text{PE}}\sqrt{d/m}$.

In conclusion, our theory does not rely essentially on Gaussianity; all key results extend to sub-Gaussian distributions with only constant-factor changes. Thus our clean and adversarial RC bounds are robust to a wide variety of realistic input distributions.

\newpage
\section{Comparison with Non-Trainable Positional Encodings}
\label{app:RoPE}
In this section, we extend our analysis to non-trainable positional encodings, specifically Rotary Positional Embeddings (RoPE). We discuss how RoPE fits into our theoretical framework and provide empirical comparisons with Trainable PE.

\subsection{Theoretical Perspective}
Our main theoretical results (Theorem \ref{thm:rc_no_pe_final} and Theorem \ref{thm:rad_bound_with_pe}) quantify the "cost of learnability" for positional information. We focus on completely trainable PE as a "worst-case" baseline because it represents the maximum flexibility in the hypothesis space. From this perspective, our results serve as a theoretical upper bound. In contrast, RoPE applies a fixed rotation to the query and key vectors. Under our function class decomposition (Eq. \eqref{eq:subadditivity}), the bias term $\mathcal{F}_{\text{bias}}$ for RoPE becomes a fixed transformation with no additional learnable parameters. Therefore, the Rademacher complexity of this bias term vanishes ($Rad_S(\mathcal{F}_{bias}) \approx 0$), and RoPE should exhibit a much tighter generalization gap than completely trainable PE, comparable to the "No-PE" baseline.

\subsection{Empirical Verification}
To validate this theoretical prediction, we conducted experiments comparing RoPE, Trainable PE, and No-PE across different context lengths $t$. The results are presented in Table \ref{tab:gap-comparison}.
\begin{table}[htbp]
\centering
\caption{Generalization Gap Comparison across Different Context Lengths $t$.}
\label{tab:gap-comparison}
\begin{tabular}{c|ccc}

\toprule
\textbf{$t$} 
& \textbf{Without PE} 
& \textbf{Trainable PE} 
& \textbf{RoPE} \\ 
\midrule
6  &  1.6149 & 2.0100 & 1.5722 \\
7  &  1.9030 & 2.0314 & 1.9081 \\
8  &  1.7055 & 2.1100 & 1.7434 \\
9  &  1.7111 & 1.8812 & 1.7406 \\
10 &  1.4730 & 1.7571 & 1.5002 \\
12 &  1.1331 & 1.2565 & 1.2060 \\
15 &  1.0547 & 1.2176 & 1.0804 \\
20 &  0.8362 & 0.9379 & 0.9969 \\
25 &  0.8601 & 0.9817 & 0.8666 \\
30 &  0.6845 & 0.8126 & 0.6610 \\
\bottomrule
\end{tabular}
\end{table}

As shown in Table \ref{tab:gap-comparison}, RoPE attains a generalization gap very similar to the "Without PE" baseline, and both are smaller than that of Trainable PE (except for an outlier at $t=20$). For example, at $t=30$, while the gap for Trainable PE remains at 0.8126, RoPE drops to 0.6610, effectively matching the Without PE baseline (0.6845). This empirically confirms our hypothesis that the widened gap observed in our main results is primarily driven by the parameter complexity of the trainable PE module.

\end{document}